\documentclass[10pt,journal,compsoc]{IEEEtran}
\usepackage[table,xcdraw]{xcolor} % 用于表格颜色
\usepackage{booktabs} % 用于表格线条
\usepackage{amsmath,amsfonts}
\usepackage{algorithm}
\usepackage{algorithmic}
\usepackage{array}
\usepackage[caption=false,font=normalsize,labelfont=sf,textfont=sf]{subfig}
\usepackage{textcomp}
\usepackage{stfloats}
\usepackage{url}
\usepackage{verbatim}
\usepackage{graphicx}
\usepackage{graphicx} 
\usepackage{amsmath}
\usepackage{amsfonts}
\usepackage{multirow}
\usepackage{color,xcolor}
\usepackage{bbding}
\usepackage{graphicx}
\usepackage{amsmath}
\usepackage{amssymb}
\usepackage{booktabs}
\usepackage{enumitem}
\usepackage{multirow}
\usepackage{tabularx}
\usepackage{subfig}

\definecolor{lightgray}{gray}{.9}
\definecolor{forestgreen}{RGB}{0,128,69}
\definecolor{cvprpink}{RGB}{219,112,147}
\definecolor{cvpryellow}{RGB}{255, 191, 0}
\definecolor{cvprblue}{rgb}{0.21,0.49,0.74}

\usepackage[colorlinks,
            linkcolor=red,       %%修改此处为你想要的颜色
            anchorcolor=blue,  %%修改此处为你想要的颜色
            citecolor=blue,        %%修改此处为你想要的颜色，例如修改blue为red
            ]{hyperref}

\usepackage{amsthm} % 如果要使用proof语句就必须要引入这个包
\newtheorem{theorem}{Theorem}

\usepackage{orcidlink}

\usepackage{bm}

\usepackage[export]{adjustbox}

\usepackage{threeparttable}
\newcolumntype{I}{!{\vrule width 1pt}}
\makeatletter
\newcommand{\thickhline}{%
    \noalign {\ifnum 0=`}\fi \hrule height 1pt
    \futurelet \reserved@a \@xhline
}
\makeatother
\usepackage{float}
\usepackage{pifont}
\usepackage{ragged2e}

\ifCLASSOPTIONcompsoc
  \usepackage[nocompress]{cite}
\else
  \usepackage{cite}
\fi
\usepackage{enumitem}
\begin{document}
%\title{Calibrating Biased Class Distributions via Cross-Domain Geometric Consistency in VFM Latent Space}
% \title{Flow-Modulated Scoring for Knowledge Graph Relation Prediction}
\title{Flow-Modulated Scoring for Semantic-Aware Knowledge Graph Completion}
%\iffalse
\author{Siyuan Li~\orcidlink{0009-0005-0035-5295},
        Ruitong Liu~\orcidlink{0000-0003-3354-9617},
		Yan Wen~\orcidlink{0000-0003-3354-9617},
        Te sun~\orcidlink{0000-0003-3354-9617},
        Andi Zhang~\orcidlink{0000-0002-5669-9354}
        Yanbiao Ma~\orcidlink{0000-0002-8472-1475},
        Xiaoshuai Hao~\orcidlink{0000-0002-5669-9354} \\
        %Zhiwu Lu~\orcidlink{0000-0001-6429-7956},~\IEEEmembership{Senior Member,~IEEE,}
        %Junchi Yan~\orcidlink{0000-0001-9639-7679},~\IEEEmembership{Senior Member,~IEEE}
        %Wenping Ma~\orcidlink{0000-0001-8872-2195},~\IEEEmembership{Senior Member,~IEEE,}
        %Shuyuan Yang~\orcidlink{0000-0002-4796-5737},~\IEEEmembership{Senior Member,~IEEE,}
        %Puhua Chen~\orcidlink{0000-0001-5472-1426},~\IEEEmembership{Senior Member,~IEEE}
        %}% <-this % stops a space
%\iffalse
%\thanks{This work was supported in part by the Key Scientific Technological Innovation Research Project of the Ministry of Education, the Joint Funds of the National Natural Science Foundation of China (U22B2054), the National Natural Science Foundation of China (62076192, 61902298, 61573267, 61906150, and 62276199), the 111 Project, the Program for Cheung Kong Scholars and Innovative Research Team in University (IRT 15R53), the Science and Technology Innovation Project from the Chinese Ministry of Education, the Key Research and Development Program in Shaanxi Province of China (2019ZDLGY03-06), and the China Postdoctoral Fund (2022T150506). 

%\emph{(Corresponding author: Zhiwu Lu.)}}% <-this % stops a space
\thanks{Siyuan Li, Ruitong Liu and Te Sun are with the Dalian University of Technology, China. Yan Wen is with the Beijing Institute of Technology, China. Andi Zhang is with the University of Manchester, U.K. Yanbiao Ma is with the Gaoling School of Artificial Intelligence, Renmin University of China. Xiaoshuai Hao is with the Beijing Academy of Artificial Intelligence, China.}% <-this % stops a space
\thanks{Correspondence author: Yanbiao Ma, Xiaoshuai Hao}
%\thanks{E-mail: ybma1998xidian@gmail.com, luzhiwu(@)ruc.edu.cn, yanjunchi@sjtu.edu.cn}
\thanks{E-mail: ybma1998xidian@gmail.com, haoxiaoshuai714@163.com)}
%\fi
}
%\fi

\iffalse
\IEEEcompsocitemizethanks{\IEEEcompsocthanksitem M. Shell was with the Department
of Electrical and Computer Engineering, Georgia Institute of Technology, Atlanta,
GA, 30332.\protect\\
E-mail: see http://www.michaelshell.org/contact.html
\IEEEcompsocthanksitem J. Doe and J. Doe are with Anonymous University.}% <-this % stops a space
\thanks{Manuscript received April 19, 2005; revised August 26, 2015.}
\fi

%\iffalse
%\markboth{TPAMI - Short Paper Submission}%
%{Shell \MakeLowercase{\textit{et al.}}: Bare Advanced Demo of IEEEtran.cls for IEEE Computer Society Journals}
%\fi

\IEEEtitleabstractindextext{%
\begin{abstract}
Knowledge graph completion demands effective modeling of multifaceted semantic relationships between entities. Yet, prevailing methods, which rely on static scoring functions over learned embeddings, struggling to simultaneously capture rich semantic context and the dynamic nature of relations. To overcome this limitation, we propose the Flow-Modulated Scoring (FMS) framework, conceptualizing a relation as a dynamic evolutionary process governed by its static semantic environment. FMS operates in two stages: it first learns context-aware entity embeddings via a Semantic Context Learning module, and then models a dynamic flow between them using a Conditional Flow-Matching module. This learned flow dynamically modulates a base static score for the entity pair. By unifying context-rich static representations with a conditioned dynamic flow, FMS achieves a more comprehensive understanding of relational semantics. 
Extensive experiments demonstrate that FMS establishes a new state of the art across both canonical knowledge graph completion tasks: relation prediction and entity prediction. On the standard relation prediction benchmark FB15k-237, FMS achieves a near-perfect MRR of 99.8\% and Hits@1 of 99.7\% using a mere 0.35M parameters, while also attaining a 99.9\% MRR on WN18RR. Its dominance extends to entity prediction, where it secures a 25.2\% relative MRR gain in the transductive setting and substantially outperforms all baselines in challenging inductive settings. By unifying a dynamic flow mechanism with rich static contexts, FMS offers a highly effective and parameter-efficient new paradigm for knowledge graph completion.
Code published at: \url{https://github.com/yuanwuyuan9/FMS}.
% Extensive experiments show FMS consistently surpasses state-of-the-art methods on multiple benchmarks. FMS sets new records on FB15k-237 with an MRR of $99.8\%$ and a Hits@1 of $99.7\%$, and on WN18RR with an MRR of $99.9\%$. These top-tier results are achieved using only $0.35M$ parameters (on FB15k-237), confirming the framework's outstanding effectiveness and efficiency.
\end{abstract}

}

\maketitle

\IEEEdisplaynontitleabstractindextext

\IEEEpeerreviewmaketitle

\iffalse
\ifCLASSOPTIONcompsoc
\IEEEraisesectionheading{\section{Introduction}\label{sec:introduction}}
\else
\section{Introduction}
\label{sec:introduction}
\fi
\fi

\section{Introduction}

Knowledge Graphs (KGs), as a robust formalism for knowledge representation, have become pivotal in a myriad of large-scale applications, including recommender systems~\cite{cui2025RS,chen2025data}, information extraction~\cite{zhang2025comprehensive}, and question answering~\cite{omar2023universal,tkde2}, owing to their remarkable efficacy in organizing and representing complex semantic facts~\cite{nicholson2020constructing,hogan2021knowledge,dong2023hierarchy,survey4kg2024}. However, real-world KGs, despite their vast scale, are ubiquitously plagued by incompleteness~\cite{tkde1,ji2021survey,tpami1}. This issue stems from the continuous evolution of knowledge itself and the high cost and complexity associated with its acquisition. To address this fundamental challenge, the task of Knowledge Graph Completion (KGC) has emerged, which aims to infer and predict missing links based on existing facts~\cite{tkde3,sun2019rotate,tpami2,tkde7}.

\begin{figure}[t]
%\vskip -0.1in
    \centering
    \includegraphics[width=1.0\linewidth]{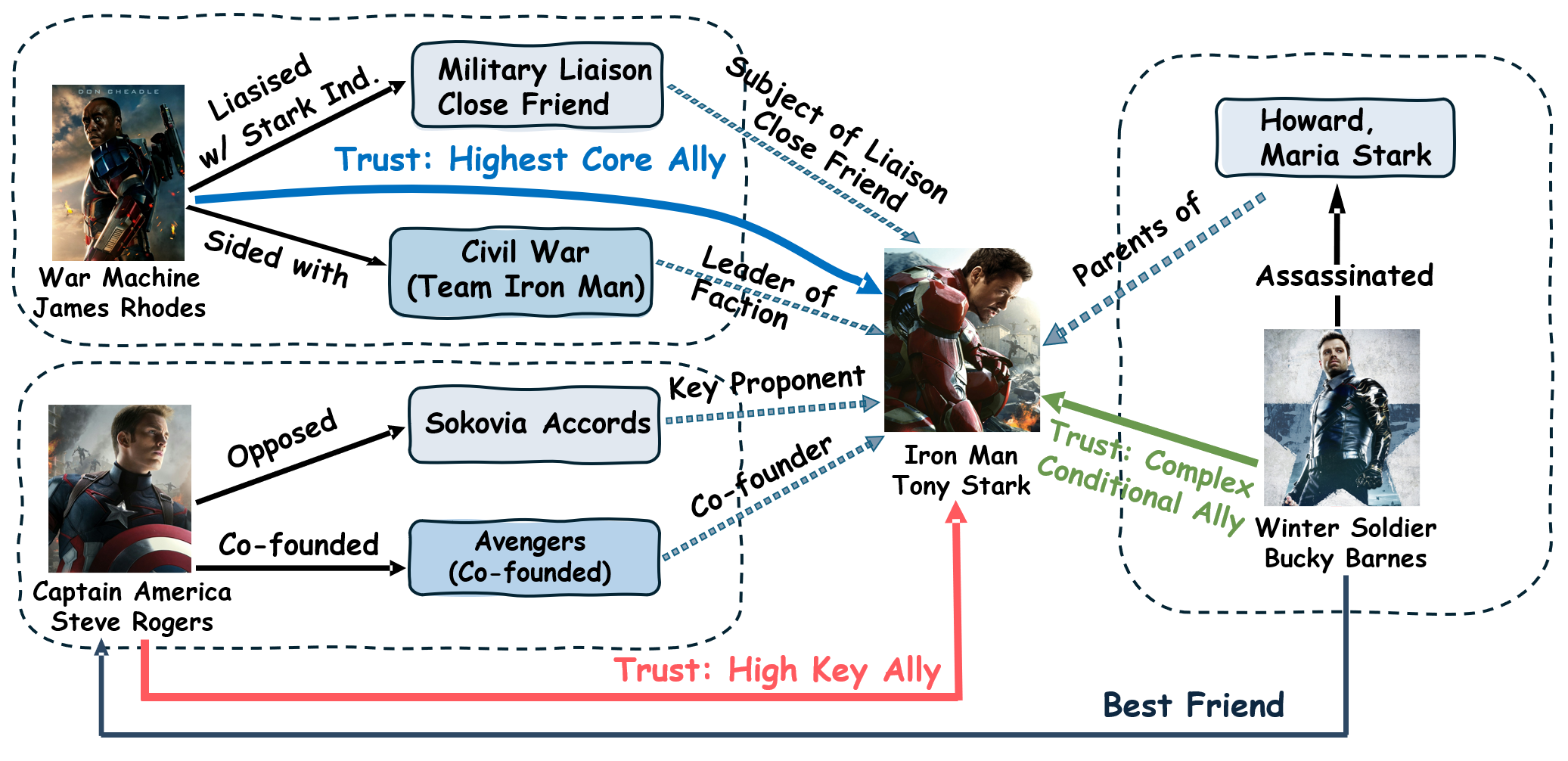}
    \vskip -0.2in
    \caption{The context-dependent nature of relational semantics, exemplified by Iron Man's alliances. The generic relation "ally" acquires distinct meanings (\textit{e.g.}, Highest Core, High Key, Complex Conditional) based on specific contextual facts, such as shared history or critical conflicts.}
    \label{fig:iron_man_relations}
    \vskip -0.1in
\end{figure}

\begin{figure}[t]
    \centering
    \includegraphics[width=1.0\linewidth]{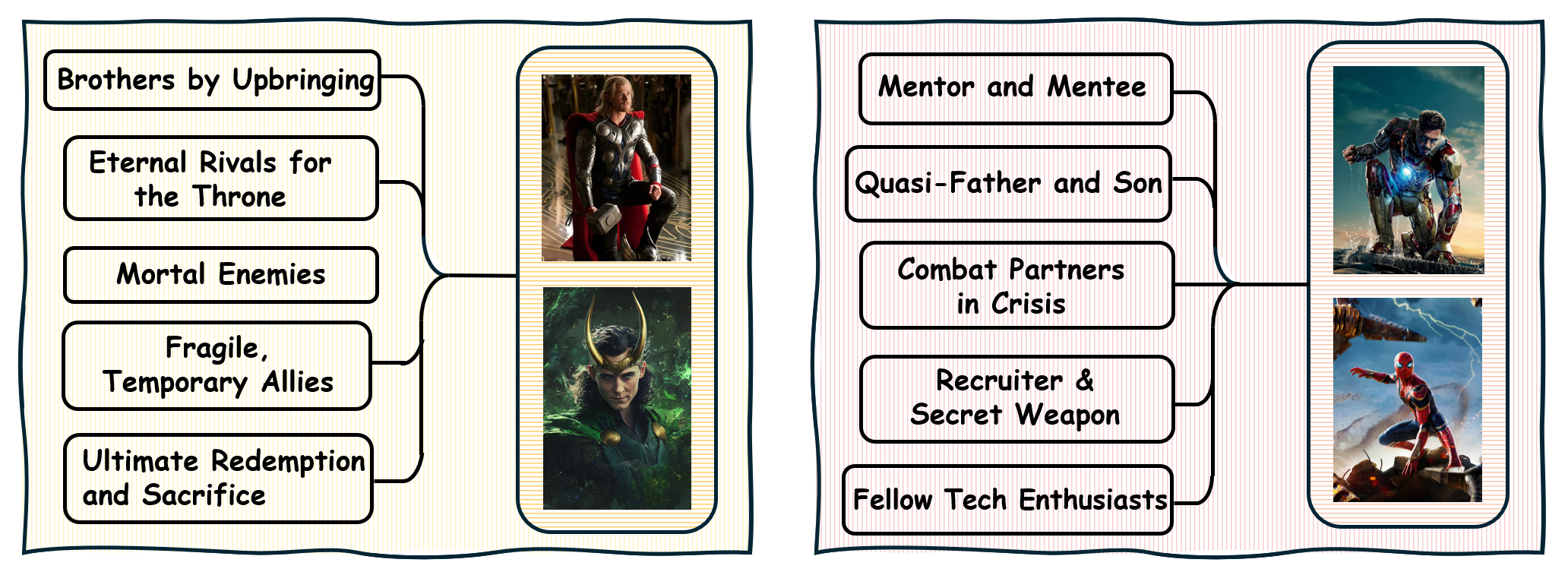}
    \vskip -0.15in
    \caption{Illustration of the dynamic evolution of relationships. The examples show the changing dynamics over time between (left) Thor and Loki, and (right) Iron Man and Spider-Man, highlighting that relationships progress through distinct stages.}
    \label{fig:dynamic_relation}
\vskip -0.1in
\end{figure}

Among the sub-tasks of KGC, Relation Prediction (RP) is crucial for enabling fine-grained knowledge reasoning~\cite{teru2020inductive,tpami3}. It is dedicated to inferring the set of potential relations $\mathcal{R}$ for a given entity pair $(h, t)$~\cite{wang2021relational}. Existing methodologies for this task predominantly fall into two categories. The first category comprises embedding-based models, such as TransE \cite{bordes2013translating}, TransH \cite{wang2014knowledge}, and DistMult \cite{yang2014embedding}, which assess the plausibility of a triplet by learning static entity and relation embeddings within a fixed scoring function. The second category consists of GNN-based methods that aggregate contextual information \cite{jagvaral2020path,tkde4}, like PathCon \cite{wang2021relational}, which leverages multi-hop neighborhood structures to enrich semantic representations. Despite the progress achieved, these methods are constrained by two fundamental bottlenecks:

\textbf{1) Inability to Capture Rich Contextual Information Shaping Relation Semantics.} The meaning of a relation is often highly contingent on its surrounding factual context~\cite{wang2021relational,aaai25RP,tpami2}. As illustrated in \autoref{fig:iron_man_relations}, the generalized label ``ally" is insufficient to accurately describe the relationships between Iron Man and other characters. His relationship with War Machine, rooted in their military liaison, is a ``core, top-level alliance," whereas his relationship with Captain America is a "key, high-level alliance" shaped by their shared history of team-building and subsequent political disagreements. This indicates that effective relation prediction must account for the subtle contextual nuances that differentiate superficially similar yet substantially distinct relations.
\textbf{2) Incapacity of Static Paradigms to Model the Dynamic Evolution of Relations.} In reality, relationships exhibit path-dependent and time-evolving characteristics~\cite{survey4kg2024,zhang2022knowledge}. For instance, the Thor-Loki dynamic shifted from brotherhood to enmity and later to reconciliation, while the Iron Man-Spider-Man bond evolved from a mentor-protégé dynamic to a quasi-paternal one (see \autoref{fig:dynamic_relation}). Accurately modeling such evolutionary trajectories is critical, a task for which existing static scoring functions are ill-equipped. While GNNs introduce structural context, they operate on fixed snapshots of the graph and are thus inherently static in their modeling approach~\cite{wang2021relational}. 

Recently, Large Language Models (LLMs) have been introduced to enhance semantic understanding in KGC~\cite{zhang2024making,tkde1,cheng2024LLM4KG}, as seen in models like KG-BERT \cite{yao2019kg} and MuKDC \cite{Li2024ijcaiLLM4KG}. Although these models offer improved context awareness, they still lack explicit mechanisms for modeling the evolutionary paths of relations and are often hampered by large parameter counts, low efficiency, and challenges in domain adaptation. Furthermore, approaches like MuKDC are also constrained by their strong reliance on meticulously designed prompts\cite{survey4kg2024,quintero2024integrating}.

To this end, this paper introduces the \textbf{Flow-Modulated Scoring (FMS)} framework, a novel approach that pioneers the unification of context-aware and dynamic evolution modeling for knowledge graph completion. The core idea of FMS is to conceptualize a relation as a conditional vector flow from a head entity $h$ to a tail entity $t$, where the trajectory of this flow is dynamically modulated by the local semantic context. Specifically, the FMS framework comprises two key modules:
1) Semantic Context Learning Module employs an energy function to identify and aggregate the top-$k$ most semantically relevant contextual paths, thereby constructing context-aware entity representations.
2) Conditional Flow Matching Module learns a conditional vector field that explicitly models the dynamic evolutionary trajectory of a relation from $h$ to $t$, enabling a path-sensitive representation of relational semantics.

By integrating these components, FMS dynamically adjusts the output of conventional scoring functions, facilitating a fine-grained modeling of relation semantics that is both context-sensitive and temporally aware. Extensive experiments demonstrate that FMS achieves state-of-the-art (SOTA) performance across multiple standard benchmarks while maintaining superior parameter efficiency.
The main contributions of this paper are summarized as follows:
\begin{itemize}
    \item \textbf{Flow-Modulated Scoring Framework:} We propose the first KGC framework that integrates context-aware and dynamic evolution modeling, overcoming the limitations of static scoring paradigms (\textcolor{red}{Section} \ref{sec3.2}).
    \item \textbf{Semantic Context Learning:} We design an energy-based top-$k$ context selection mechanism to efficiently capture critical semantic dependencies (\textcolor{red}{Section} \ref{sec3.2.1}).
    \item \textbf{Conditional Flow Matching:} We introduce a conditional vector flow mechanism to explicitly model the evolutionary paths of relations, enhancing the capacity for dynamic semantic representation (\textcolor{red}{Section} \ref{sec3.2.2}).
    \item We demonstrate the state-of-the-art performance and high parameter efficiency of FMS on multiple benchmarks, confirming its effectiveness in modeling complex relations (\textcolor{red}{Section} \ref{sec4} and \textcolor{red}{Section} \ref{sec5}).
\end{itemize}

%%================================%%
%%                         Related Work                          %%
%%================================%%
\section{Related Work}

\subsection{Knowledge Graph Completion}

Knowledge Graphs provide structured information for downstream tasks such as recommender systems and semantic analysis \cite{tkde3,survey4kg2024,zhao2025kg4re}. Knowledge Graph Completion (KGC), and particularly Relation Prediction (RP), addresses the problem of KG incompleteness by predicting missing links \cite{tkde1,ibrahim2024survey,tkde5,10948338}.

% The majority of KGC methods rely on embedding-based approaches, which assign vectors to entities and relations in a continuous space and are trained on observed facts~\cite{long2024kgdm,long2024fact,ge2024knowledge,cao2024knowledge}. For instance, TransE~\cite{bordes2013translating}, RotaE~\cite{sun2019rotate} treats entities as points and relations as translations, aiming to position the translated head entity close to the tail entity in real, complex, or quaternion spaces. Additionally, bilinear models like DistMult~\cite{yang2014embedding} and ComplEx~\cite{trouillon2016complex} compute semantic similarity via matrix or vector dot products. Some methods have explored advanced architectures beyond point vectors, such as CNN-based models~\cite{dettmers2018convolutional,jagvaral2020path,tkde4}.
The majority of KGC methods rely on embedding-based approaches, which assign vectors to entities and relations in a continuous space and are trained on observed facts~\cite{long2024kgdm,long2024fact,ge2024knowledge,cao2024knowledge}. For instance, early geometric and bilinear models, such as TransE~\cite{bordes2013translating} and ComplEx~\cite{trouillon2016complex}, laid the groundwork by defining simple yet effective scoring functions. To capture more intricate relational patterns, subsequent research has advanced into more expressive geometric spaces. RotaE~\cite{sun2019rotate}, for example, introduced relational rotations in complex space, a concept later enhanced by Rot-Pro~\cite{song2021rotpro} for greater fidelity. Other approaches have explored specialized geometries, such as HAKE's~\cite{zhang2020learning} use of polar coordinates for semantic hierarchies and BoxE's~\cite{Boxe2020} hyper-rectangle representations for logical conjunctions. In a parallel thread of research, recent work has focused not on geometry but on augmenting the reasoning process itself. AnKGE~\cite{AnKGE2023} exemplifies this direction by equipping base models with an analogical inference mechanism to leverage similar, known facts~\cite{tkde4}.

Despite their effectiveness in handling relational patterns, these models are fundamentally static. They rely on pre-computed, fixed embeddings, which renders the models context-agnostic; the representation of an entity or relation remains constant regardless of the specific factual query \cite{survey4kg2024}. This paradigm is inherently incapable of capturing the subtle nuances and contextual dependencies of relational meanings.
To overcome static embeddings, recent research has leveraged Large Language Models (LLMs) to inject contextual awareness \cite{tan2025paths,10948338}. These approaches either refine link prediction as a textual scoring task (e.g., KG-BERT~\cite{yao2019kg}, SimKGC~\cite{wang2022simkgc}) or utilize LLMs as generative knowledge sources to distill missing facts, as exemplified by MuKDC~\cite{Li2024ijcaiLLM4KG}. Despite their promise, these LLM-based methods are often hampered by their massive parameter counts, low computational efficiency, and significant challenges in domain adaptation \cite{10948338}.
Another direction to overcome static embeddings is through generative modeling. Recent works like FDM~\cite{long2024fact} and KGDM~\cite{Long_Zhuang_Li_Wei_Li_Wang_2024} employ diffusion models~\cite{ho2020denoising} to learn fact distributions. However, their generative process is typically conditioned solely on the local query triplet, thus failing to leverage the broader, multi-hop structural context crucial for robust link prediction.

The advent of GNN-based models represents an attempt to incorporate contextual information from the graph structure~\cite{chang2024path,tkde6}. Methods like PathCon~\cite{wang2021relational} enrich entity and relation representations by aggregating information from multi-hop paths, demonstrating the value of neighborhood context. To mitigate the high cost of explicit path extraction, recent work CBLiP~\cite{dutta2025replacing} instead uses connection-biased attention to implicitly capture structural information from the local subgraph. However, these approaches still operate on a static snapshot of the graph. While they encode structural context, they do not explicitly model the formation process or evolutionary trajectory of relations. Their aggregation mechanisms are typically fixed and cannot adapt to the dynamic and path-dependent characteristics inherent in how relationships form and evolve. In summary, prevailing methodologies exhibit an over-reliance on static modeling paradigms, making them ill-equipped to handle the dual challenges of context-sensitivity and dynamic evolution.

\subsection{Modeling Dynamic Processes with Conditional Flow Matching}

To address the core challenges of modeling dynamic relationship evolution, we draw on research in generative models concerning the learning of transformations and trajectories between data distributions. Although diffusion models~\cite{ho2020denoising,song2020score} have demonstrated formidable generative capabilities in domains such as computer vision~\cite{li2022srdiff} and sequence modeling~\cite{li2022diffusion,tashiro2021csdi}, their training and inference processes are often computationally expensive.

A more direct and efficient approach for learning transformation paths is offered by Continuous Normalizing Flows (CNFs)~\cite{chen2018neural}. CNFs learn an invertible mapping between a data distribution and a simple prior distribution by parameterizing the vector field of an Ordinary Differential Equation (ODE). Recently, Flow Matching (FM)~\cite{lipman2023flow, albergo2023stochastic, liu2022rectified,geng2025mean} has emerged as a highly efficient and stable technique for training CNFs. Unlike score-matching methods that learn the gradient of the data distribution's logarithm, FM learns the transformation by directly regressing the ODE's vector field. This simulation-free training approach significantly enhances model stability and efficiency. The advent of FM introduced new paradigms for CNF training and improved sample quality, and while early FM methods often assumed a Gaussian source distribution, this limitation has been addressed in subsequent work~\cite{chen2024flowgeom, pooladian2023multisample}.

Building on this foundation, Conditional Flow Matching (CFM)~\cite{tong2023improving,ma2024sit,gat2024discrete} offers a more generalized and powerful framework. CFM aims to learn probability paths conditioned on specific attributes, providing a unified, simulation-free training objective for models with arbitrary transport maps. Crucially, it neither requires the source distribution to be Gaussian nor necessitates the evaluation of its probability density. These properties—efficiency, stability, and flexible conditioning—make CFM the ideal theoretical tool for our objectives.
\textbf{In this work}, we innovatively apply the core mechanism of Conditional Flow Matching within KGC. Instead of using it for data generation, we leverage it to model the \textbf{conditional evolutionary path} of a relation from a head entity $h$ to a tail entity $t$. This enables us to explicitly model relations as trajectories, dynamically modulated by semantic context, thereby overcoming the static limitations of existing KGC methods.

%%================================%%
%%  Foundation Model as Bridges for Transferring Geometric Knowledge                %%
%%================================%%

\begin{figure*}[t]
    \centering
    \includegraphics[width=0.98\linewidth]{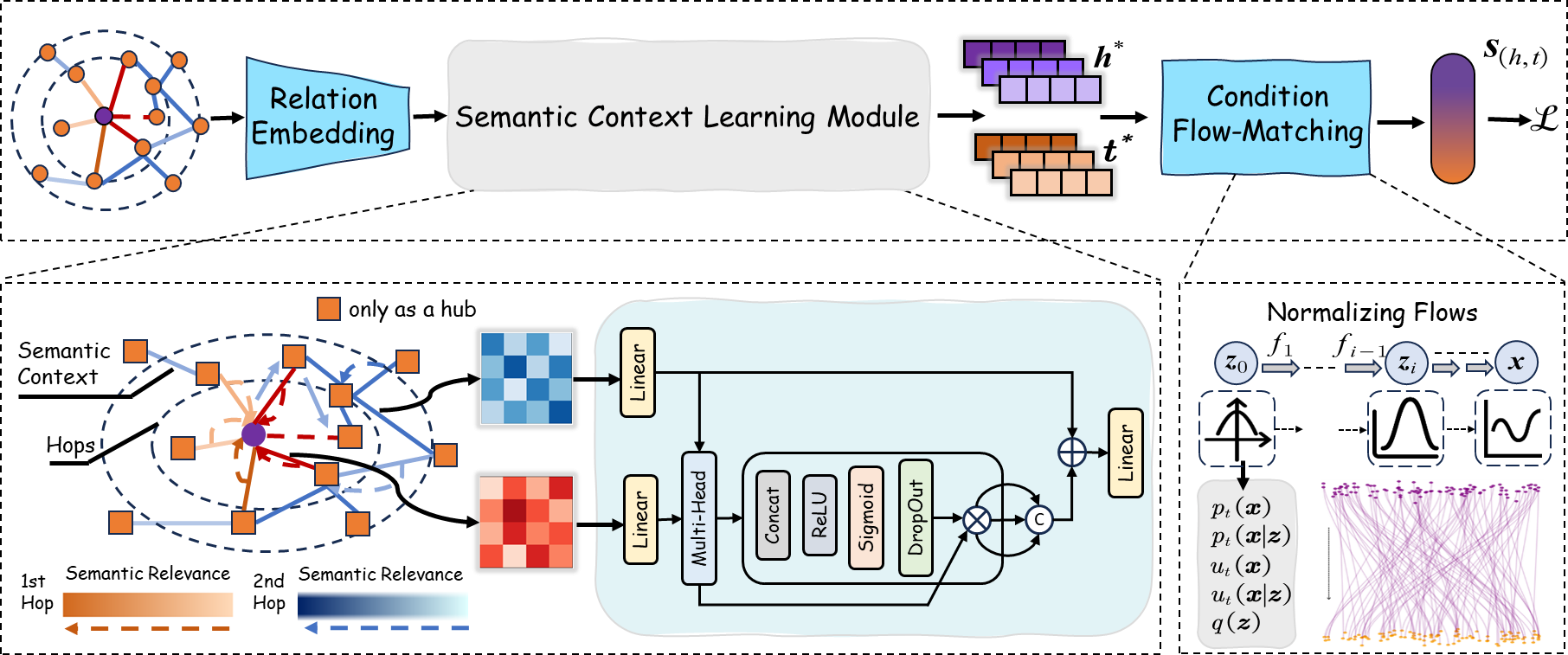}
    \caption{The architecture of the Flow-Modulated Scoring (FMS) framework. In Semantic Context Learning Module, nodes(square shapes) do not consider their own intrinsic node messages; rather, their role is to facilitate the propagation and aggregation of messages associated with their incident edges, which represent the relation semantics. Conditional Flow Matching Module learns a dynamic flow that modulates the static score to produce the final prediction, unifying static context with dynamic evolution.}
    \label{fig:architecture}
\end{figure*}

\section{Methodology}
\label{sec3}
Relational links arise from the interplay of two fundamental principles: the static semantic context of entities and the dynamic evolutionary process connecting them. In this section, we formalize this duality through a probabilistic decomposition of relation prediction (\textcolor{red}{Section} \ref{sec3.1}), which in turn motivates our two-component approach. Subsequently, we introduce the core components of the Flow-Modulated Scoring model (\textcolor{red}{Section} \ref{sec3.2}) and show how they are unified into a single, cohesive scoring function under a joint training objective (\textcolor{red}{Section} \ref{sec3.2.3}).

\subsection{Preliminaries}
\label{sec3.1}
Let \( \mathcal{G} = (\mathcal{V}, \mathcal{E}) \) be an instance of a knowledge graph, where \( \mathcal{V} \) is the set of entity nodes and \( \mathcal{E} \) is the set of edges, with each edge representing a relation \( r \in \mathcal{R} \)~\cite{zhang2022knowledge}. Our goal is to predict missing relations in the knowledge graph, i.e., given an entity pair \( (h, t) \), infer the relation \( r \) between them. 

Specifically, we aim to model the distribution of relation types given the entity pair \( (h, t) \), denoted as \( p(r|h, t) \), to capture the complex semantic context and dynamic nature of relations. According to Bayes' theorem, this distribution can be expressed as:
\begin{equation}
    \label{eq:bayes}
    p(r|h, t) \propto p(h, t|r) \cdot p(r),
\end{equation}
where \( p(r) \) is the prior distribution over relation types, serving as a regularization term. To further model the joint probability \( p(h, t|r) \), we decompose it into a symmetric form:
\begin{equation}
    \label{eq:decomposition}
p(h, t|r) = \frac{1}{2} \left( p(h|r) \cdot p(t|h, r) + p(t|r) \cdot p(h|t, r) \right).
\end{equation}

Equation~(\ref{eq:decomposition}) provides guidance for our modeling approach. The terms \( p(h|r) \) and \( p(t|r) \) measure the likelihood of entities \( h \) or \( t \) given the relation \( r \), reflecting the semantic context of the entities. The terms \( p(t|h, r) \) and \( p(h|t, r) \) capture the dynamic evolution from \( h \) to \( t \) or from \( t \) to \( h \) given the relation \( r \), reflecting the dynamic nature of relations. The subsequent sections detail how our proposed model instantiates these two factors and integrates them for Knowledge graph completion.

\subsection{Flow-Modulated Scoring Model}
\label{sec3.2}
The overall framework of the Flow-Modulated Scoring model encompasses \textbf{two core components:} Semantic Context Learning and Conditional Flow Matching. As illustrated in \autoref{fig:architecture}, the semantic context learning stage is designed to encode context-aware entity embeddings, while the conditional flow matching stage learns a dynamic flow from the head entity embedding to the tail entity embedding. This dynamic flow is then used to modulate a static score for the entity pair, with the ultimate goal of achieving comprehensive relation prediction. Next, we provide a detailed explanation of these components.

\subsubsection{Semantic Context Learning}
\label{sec3.2.1}

For a given triplet $(h, r, t)$ in a Knowledge Graph (KG), its semantic context offers valuable cues for identifying valid links~\cite{wang2021relational}. Traditional node-based message passing, as seen in GraphSAGE~\cite{hamilton2017inductive}, GAT~\cite{vaswani2017attention} and RED-GNN~\cite{zhang2022knowledge}, is applicable to general graphs. However, when applied to KGs, we observe that edges themselves possess rich features (i.e., relation embeddings $x_e$), and the number of relation types is significantly smaller than the number of entities. Consequently, modeling message passing directly between relations (edges) is not only more efficient but also provides more indicative semantic cues. As illustrated in \autoref{fig:iron_man_relations}, our reasoning about Iron Man's "ally" status is predicated on these relational attributes, rather than on the intrinsic properties of the entities themselves.

To this end, we propose a relational message passing scheme where edges update their states by aggregating information from neighboring edges. This approach treats the KG's relational structure as the primary computational graph. However, naively aggregating information from all neighboring edges can introduce noise and lead to information dilution or over-smoothing, especially when some neighbors are semantically irrelevant to the target relation.

To address this challenge, we introduce a \textit{\textbf{Top-K selection mechanism based on semantic similarity between edges.}} Specifically, when updating the state of a central target edge $e_c$, we first evaluate its semantic relevance to each neighboring edge $e_n$ in its local neighborhood $\mathcal{N}(e_c)$. This is achieved by projecting their current states, $s^{e_c}_i$ and $s^{e_n}_i$, into a shared latent semantic space using a learnable mapping function $g(\cdot)$~\cite{cao2024diffusione}. In this space, we compute their similarity score using an energy-model inspired scoring function:
\begin{equation}
    \text{Score}(e_c, e_n) = \exp(-\|g(s^{e_c}_i) - g(s^{e_n}_i)\|^2 / \tau),
    \label{eq:semantic_score}
\end{equation}
where $\| \cdot \|^2$ is the squared Euclidean distance and $\tau$ is a temperature hyperparameter. Edges with higher scores, deemed more semantically relevant to the central edge $e_c$, are selected to form the Top-K set $\mathcal{N}_K(e_c) \subseteq \mathcal{N}(e_c)$.

The states of these selected Top-K edges, $\{ s^{e_n}_i \}_{e_n \in \mathcal{N}_K(e_c)}$, are first mean-aggregated to form a consolidated neighborhood representation:
\begin{equation}
    \bar{s}^{\mathcal{N}_K}_i(e_c) = \text{Mean}\left( \{ s^{e_n}_i \}_{e_n \in \mathcal{N}_K(e_c)} \right).
    \label{eq:mean_agg_edges}
\end{equation}

This aggregated neighborhood representation $\bar{s}^{\mathcal{N}}_i(e_c)$, along with the central edge's own state $s^{e_c}_i$, is then processed by our specialized attention aggregator to compute the updated edge state for the next layer, $s^{e_c}_{i+1}$. This aggregator employs a multi-head attention mechanism that functions as a gating system. It transforms both $s^{e_c}_i$ and $\bar{s}^{\mathcal{N}}_i(e_c)$, calculates head-wise gating scores based on their interaction, applies these gates to the transformed neighborhood representation, and then combines this gated information with the transformed self-representation before a final projection. This process yields a semantically focused update:
\begin{equation}
    s^{e_c}_{i+1} = \text{AttentionAggregator}\left( s^{e_c}_i, \bar{s}^{\mathcal{N}_K}_i(e_c) \right).
    \label{eq:edge_state_update}
\end{equation}
The initial state of an edge is defined by its intrinsic feature embedding, $s^e_0 = x^e \in \mathbb{R}^{d_e}$. This iterative update is performed for a predefined number of layers. This process ultimately generates a contextual representation for the central entity, which encapsulates the relational information within its multi-hop neighborhood. The impact of the semantic scoring function and aggregator choice will be further discussed in our ablation studies (\autoref{sec4.4}).

In this manner, our model directly leverages the relational structure of the KG. Through the semantic-aware edge neighbor selection and the subsequent attention-based aggregation, the model accurately captures and propagates context information most relevant to the specific relation prediction task, performing reasoning directly in the relation space.Notably, this process avoids the substantial computational overhead of encoding entity embeddings by not leveraging any entity-specific features.

\subsubsection{Condition Flow-Matching Module} 
\label{sec3.2.2}
Prior research typically predicts static association scores directly utilizing entity embeddings. Although our Semantic Context Learning module is able to generate more informative entity representations, it is still insufficient to characterize the dynamics of relationships. \textbf{To address this limitation}, we have designed the Conditional Flow Matching Module. It constructs a conditional flow-based flow matching mechanism, incorporating the semantic context in which entities are situated as conditional variables into the relationship modeling process.

The core of this module is learning a dynamic vector field
\(\boldsymbol{v}_{\theta}(t,\boldsymbol{x})\) that depends on a pseudo-time variable \(t \in [0,1]\) and the current state {\( \boldsymbol{\mathit{x}} \in \mathbb{R}^{\scriptstyle \mathit{d}} \)}. The reason behind our proposal of the pseudo-time variable lies is the fact that the dynamics of relationships are inherently intertwined with the temporal dimension—evolution of relationships (in the real world it is mostly accompanied by the passage of time). However, the construction of traditional knowledge graph datasets has limitations: they do not explicitly incorporate temporal annotation information. Instead, such temporal attributes are latent in the overall distribution of massive triple data, making them difficult to be directly captured and utilized. To address this, we introduce $t$ as a bridge. It does not directly replicate physical time, but rather externalizes the originally concealed temporal properties through parameterization. This treatment makes the model to more accurately capture the intrinsic logic of relationship evolution over "time". This learned vector field characterizes the instantaneous direction and velocity along the evolution path from the head to the tail entity embedding.

Specifically, drawing upon the principles of Conditional Flow Matching~\cite{tong2023improving}, we employ a continuous flow to model the evolution from the head entity embedding \(\boldsymbol{\mathit{h}}^{*}\) to the tail entity embedding \(\boldsymbol{\mathit{t}}^{*}\)\footnote{The asterisk (*) is used here to distinguish entity representations from the pseudo-time variable \(t\). This notation will be maintained without repeated clarification.} This process can be formalized as:
\begin{equation}
\mathit{p}_{\mathit{t}}(\boldsymbol{\mathit{x}}) = \int \mathit{p}_{\mathit{t}}(\boldsymbol{\mathit{x}}|\boldsymbol{\mathit{z}}) \mathit{q}(\boldsymbol{\mathit{z}}) \mathit{d}\boldsymbol{\mathit{z}},
\label{(8)}
\end{equation}
where \(\mathit{p_t(\boldsymbol{\mathit{x}})}\) represents the marginal probability path, \(q(\boldsymbol{\mathit{z}})\) is the sampling distribution of the semantic context \(\boldsymbol{z}\), and the conditional probability path \(p_{t}(\boldsymbol{\mathit{x}}|\boldsymbol{\mathit{z}})\) denotes the probability path given \(\boldsymbol{z}\). Correspondingly, the marginal vector field \(u_t(\boldsymbol{\mathit{x}})\) that drives the evolution of \(p_{t}(\boldsymbol{\mathit{x}})\) can be expressed as an expectation of the conditional vector field \(u_{t}(\boldsymbol{\mathit{x}}|\boldsymbol{\mathit{z}})\) with respect to the prior distribution \(q(\boldsymbol{\mathit{z}})\), weighted by the posterior distribution \(p(\boldsymbol{\mathit{z}} | t, \boldsymbol{\mathit{x}})\):
\begin{equation}
u_t(\boldsymbol{\mathit{x}}) = E_{q(\boldsymbol{\mathit{z}})}\left[ u_t(\boldsymbol{\mathit{x}}|\boldsymbol{\mathit{z}})\frac{p_t(\boldsymbol{\mathit{x}}|\boldsymbol{\mathit{z}})}{p_t(\boldsymbol{\mathit{x}})} \right],
\label{(9)}
\end{equation}
where \(u_{t}(\boldsymbol{\mathit{x}}|\boldsymbol{\mathit{z}}) = \frac{d\boldsymbol{\mathit{x}}}{d t}\). Consequently, under the initial condition \(p_{t}(\boldsymbol{\mathit{x_0}})\), \(u_{t}(\boldsymbol{\mathit{x}})\) generates the probability flow path \(p_{t}(\boldsymbol{\mathit{x}})\). We formally establish this result as follows.
\begin{proof}
\label{proof1}
To prove that the vector field \(u_t(\boldsymbol{x})\) generates the probability path \(p_t(\boldsymbol{x})\), we must show that they satisfy the continuity equation: 
$
\frac{\partial p_t(\boldsymbol{x})}{\partial t} + \nabla_{\boldsymbol{x}} \cdot \left( u_t(\boldsymbol{x}) p_t(\boldsymbol{x}) \right) = 0.
$
We begin by taking the time derivative of the marginal probability path defined in Eq.~\ref{(8)}:
\begin{equation}
    \frac{\partial p_t(\boldsymbol{x})}{\partial t} = \frac{\partial}{\partial t} \int p_t(\boldsymbol{x}|\boldsymbol{z})q(\boldsymbol{z})\, d\boldsymbol{z} \\
    = \int \frac{\partial p_t(\boldsymbol{x}|\boldsymbol{z})}{\partial t} q(\boldsymbol{z}) \, d\boldsymbol{z},
\label{(22)}
\end{equation}
where the second of these equal signs is guaranteed by the Leibniz Rule. By definition, the conditional path \(p_t(\boldsymbol{x}|\boldsymbol{z})\) is generated by the conditional vector field \(u_t(\boldsymbol{x}|\boldsymbol{z})\). Therefore, it satisfies its own continuity equation:
\begin{equation}
    \frac{\partial p_t(\boldsymbol{x}|\boldsymbol{z})}{\partial t} = -\nabla_{\boldsymbol{x}} \cdot \left( u_t(\boldsymbol{x}|\boldsymbol{z}) p_t(\boldsymbol{x}|\boldsymbol{z}) \right).
\end{equation}
Substituting this into our expression gives:
\begin{align}
    \frac{\partial p_t(\boldsymbol{x})}{\partial t} &= \int -\nabla_{\boldsymbol{x}} \cdot \left( u_t(\boldsymbol{x}|\boldsymbol{z}) p_t(\boldsymbol{x}|\boldsymbol{z}) \right) q(\boldsymbol{z}) \, d\boldsymbol{z} \\
    &= -\nabla_{\boldsymbol{x}} \cdot \int u_t(\boldsymbol{x}|\boldsymbol{z}) p_t(\boldsymbol{x}|\boldsymbol{z}) q(\boldsymbol{z}) \, d\boldsymbol{z}.
\end{align}
In the second step, we moved the divergence operator \(\nabla_{\boldsymbol{x}} \cdot\) outside the integral, as the integration is with respect to \(\boldsymbol{z}\).
Now, let's examine the definition of the marginal vector field \(u_t(\boldsymbol{x})\) from Eq.~\ref{(9)}. By multiplying both sides by \(p_t(\boldsymbol{x})\), we get:
\begin{equation}
    u_t(\boldsymbol{x}) p_t(\boldsymbol{x}) = \int u_t(\boldsymbol{x}|\boldsymbol{z}) p_t(\boldsymbol{x}|\boldsymbol{z}) q(\boldsymbol{z}) \, d\boldsymbol{z}.
\end{equation}
The right-hand side is precisely the term inside the integral in our derivation. Substituting this back, we obtain:
\begin{equation}
    \frac{\partial p_t(\boldsymbol{x})}{\partial t} = -\nabla_{\boldsymbol{x}} \cdot \left( u_t(\boldsymbol{x}) p_t(\boldsymbol{x}) \right).
\end{equation}
Rearranging the terms yields the continuity equation:
\begin{equation}
    \frac{\partial p_t(\boldsymbol{x})}{\partial t} + \nabla_{\boldsymbol{x}} \cdot \left( u_t(\boldsymbol{x}) p_t(\boldsymbol{x}) \right) = 0.
\end{equation}
Since the marginal probability path \(p_t(\boldsymbol{x})\) and the marginal vector field \(u_t(\boldsymbol{x})\) satisfy the continuity equation, and \(p_t(\boldsymbol{x})\) satisfies the initial condition \(p_0(\boldsymbol{x})\) by construction, it follows that \(u_t(\boldsymbol{x})\) is indeed the generator of the probability path \(p_t(\boldsymbol{x})\).
\end{proof}

Although we aim to recover \(u_{t}(\boldsymbol{\mathit{x}})\), its direct computation is generally difficult. Considering the relative tractability of the conditional vector field \(p_t(\boldsymbol{\mathit{x}}|\boldsymbol{\mathit{z}})\), we adopt an indirect learning strategy, and the model is enabled to directly learn an approximation of the vector field \(u_{t}(\boldsymbol{\mathit{x}}|\boldsymbol{\mathit{z}})\). We assume \(v_{\theta}:\mbox{\small\([0,1]\times\mathbb{R}^{d}\to\mathbb{R}^{d}\)}\) is a time-dependent vector field parameterized by a neural network with parameters \(\theta\). We define the following objective function:
\begin{equation}
\mathcal{L}_{cfm}(\theta)=\mathbb{E}_{t,q(\boldsymbol{\mathit{z}}),p_{t}(\boldsymbol{\mathit{x}}\mid\boldsymbol{\mathit{z}})}\left\|v_{\theta}(t,\boldsymbol{\mathit{x}}) - u_{t}(\boldsymbol{\mathit{x}}|\boldsymbol{\mathit{z}})\right\|_{2}^{2}.
\end{equation}
Based on this, the model learns \( v_{\theta}(t,\boldsymbol{x}) \). It can serve as a dynamic adjustment factor to modulate the static base score assigned to the pair of corresponding entities. 
 
Then, we draw on the optimal transport theory~\cite{Schrijver2003}, and aims to model the least-cost evolution path from head entity distribution to tail entity distribution. \(q(\boldsymbol{z}\)) defined as:
\begin{equation}
\\ q(\boldsymbol{z}) = \pi(\boldsymbol{z})=\underset{\pi \in \Pi(\rho_H, \rho_T)}{\text{argmin}} \int_{\mathbb{R}^d \times \mathbb{R}^d} \left\lVert \boldsymbol{x} - \boldsymbol{y} \right\rVert_2^2 \, d\pi(\boldsymbol{x}, \boldsymbol{y}),
\end{equation}
where \(\rho_H\) and \(\rho_T\) represent the marginal distributions of the head and tail entity embeddings, respectively, and \(\Pi(\rho_H, \rho_T)\) is the set of all joint probability distributions. This design aims to guide the model to learn a flow field aligned with the static optimal transport path. The conditional probability path and its corresponding conditional vector field are defined as in the paper~\cite{Schrijver2003}:
\begin{equation}
p_{t}(\boldsymbol{x}|\boldsymbol{z}) = \mathcal{N}(\boldsymbol{x}\mid(1 - t)\boldsymbol{h}^* + t\boldsymbol{t}^*, \sigma^2 \mathbf{I}),
\label{(12)}
\end{equation}
\begin{equation}
u_{t}(\boldsymbol{x}|\boldsymbol{z}) = \boldsymbol{t}^* - \boldsymbol{h}^*,
\label{(13)}
\end{equation}
where \(p_{t}(\boldsymbol{x}|\boldsymbol{z})\) is a Gaussian distribution, and the noise standard deviation  \(\sigma\) controls the stochasticity of the probability path. In particular, as $\sigma \to 0$, The marginal vector field can be regarded as a generative model of $\boldsymbol{h}^*$. Proof is given in the \hyperref[thm:example]{Theorem}. A moderate value of \(\sigma\) can enhance model robustness. Through CSM, model can learn more direct evolution paths between entities.

\begin{theorem}\label{thm:example}
The marginal $p_t$ corresponding to $q(z) = q(\boldsymbol{h}^*)q(\boldsymbol{t}^*)$ and the $p_t(x|z)$, $u_t(x|z)$ in (\ref{(12)}) and (\ref{(13)}) has boundary conditions $p_1 = q_1 \ast \mathcal{N}(x \mid 0, \sigma^2)$ and $p_0 = q_0 \ast \mathcal{N}(x \mid 0, \sigma^2)$, where $\ast$ denotes the convolution operator.
\end{theorem}
\begin{proof}
We start with (\ref{(8)}) to show the result. We note that \( q(z) = q(\boldsymbol{h}^*, \boldsymbol{t}^*)) = q(\boldsymbol{h}^*)q(\boldsymbol{t}^*) \), so :

\[
\begin{aligned}
p_t(x) &= \int p_t(x \mid z) q(z) \, dz \\
&= \int \mathcal{N}\bigl(x \mid (1 - t)\boldsymbol{h}^* + t\boldsymbol{t}^*, \sigma^2\boldsymbol{I}\bigr) q\bigl((\boldsymbol{h}^*, \boldsymbol{t}^*)\bigr) d(\boldsymbol{h}^*, \boldsymbol{t}^*) \\
&= \iint  \mathcal{N}\bigl(x \mid (1 - t)\boldsymbol{h}^* + t\boldsymbol{t}^*, \sigma^2\boldsymbol{I}\bigr) q(\boldsymbol{h}^*) q(\boldsymbol{t}^*) \, d\boldsymbol{h}^* d\boldsymbol{t}^*
\end{aligned}
\]

evaluated at \( i = 0, 1 \) respectively. Therefore, at \( t = 0 \),

\[
\begin{aligned}
p_0(x) &= \iint \mathcal{N}\bigl(x \mid\boldsymbol{h}^*, \sigma^2\bigr) q(\boldsymbol{h}^*) q(\boldsymbol{t}^*) \, d\boldsymbol{h}^* d\boldsymbol{t}^* \\
&= \int \mathcal{N}\bigl(x \mid\boldsymbol{h}^*, \sigma^2\bigr) q(\boldsymbol{h}^*) \, d\boldsymbol{h}^* \\
&= q(\boldsymbol{h}^*) * \mathcal{N}\bigl(x \mid 0, \sigma^2\bigr).
\end{aligned}
\]

This is also true for \( t = 1 \).
\end{proof}

%%================================%%
%%                        IMPORTANT TODO：增加实体预测子任务，需要修改损失函数，然后增加一个overall performance和一个case study。                         %%
%%================================%%

\subsubsection{Overall Scoring Process and Joint Training}
\label{sec3.2.3}
Having detailed the Semantic Context Learning module and the Conditional Flow-Matching module for modeling dynamic evolution, this section will now describe how they are integrated into a unified scoring and training framework. The process unfolds in two sequential steps to produce a final, context-sensitive, and dynamically-aware score.

Firstly, we utilize the Semantic Context Learning module to compute the final messages $m_{j-1}^h$ and $m_{j-1}^t$ for the head entity $h$ and the tail entity $t$. These messages summarize their respective consistent contextual semantics.
Subsequently, we represent the static score for the entity pair $(h,t)$ as follows:
\begin{equation}
    s_{(h,t)} = \sigma \left( \left[ m_{j-1}^h, m_{j-1}^t \right] \cdot W_{j-1} + b_{j-1} \right).
    \label{eq:context_representation}
\end{equation}
Note that Equation~\eqref{eq:context_representation} should only take the messages of $h$ and $t$ as input, without including their connecting edge $r$, since the ground-truth relation $r$ is treated as unobserved during the training stage.

Secondly, we employ the Condition Flow-Matching Module. Given the contextual condition $\boldsymbol{z} = (\boldsymbol{h}^*, \boldsymbol{t}^*) = (m_{j-1}^h, m_{j-1}^t)$, where the first symbol representation used here for the CFM module is for the sake of convenience in representation, this module predicts a dynamic vector field \(\boldsymbol{v}_{\theta}(t,\boldsymbol{x})\). This vector field characterizes the instantaneous velocity and direction of the evolution from an initial state $\boldsymbol{h}^*$ (i.e., $\boldsymbol{x}$ at $t=0$) to a target state $\boldsymbol{t}^*$ (i.e., $\boldsymbol{x}$ at $t=1$). We then use this vector field to dynamically modulate the basic static score of the entity pair as follows:
\begin{equation}
    s_{(h,t)} = s_{(h,t)} \odot v_{\theta}(t, \boldsymbol{x}).
\label{eq:dynamiclly_score}
\end{equation}

Given the flow-modulated score of the entity pair $s_{(h,t)}$, we can convert the score into a probability distribution over relations by applying softmax, which then allows us to predict the relations:
\begin{equation}
    p(r|h,t) = \text{SoftMax} \left( s_{(h,t)}\right).
\end{equation}

The primary objective for relation prediction is to minimize the cross-entropy loss between the predicted probabilities and the ground-truth relations over the training triplets $\mathcal{D}$. This is formulated as minimizing the negative log-likelihood of the true relations:
\begin{equation}
    \mathcal{L}_\textit{{\text{pred}}}(\Phi_S, \theta) = -\sum_{(h,r,t) \in \mathcal{D}} \log(p(r|h,t)),
    \label{eq:prediction_loss}
\end{equation}
where $\Phi_S$ represents all learnable parameters of the Semantic Context Learning module and the static scoring Linear.

To train the entire model effectively, we adopt a joint training strategy. The overall loss function $\mathcal{L}$ combines the relation prediction loss with the Condition Flow-Matching loss. The Condition Flow-Matching loss ensures that the learned vector field \(\boldsymbol{v}_{\theta}(t,\boldsymbol{x})\) accurately approximates the desired conditional dynamics.
The final training objective is to minimize the weighted sum of these losses:
\begin{equation}
    \mathcal{L}(\Phi_S, \theta) = \mathcal{L}_{\textit{\text{pred}}}(\Phi_S, \theta) + \lambda \mathcal{L}_{cfm}(\theta),
    \label{eq:total_loss}
\end{equation}
where $\lambda$ is a hyperparameter that balances the contribution of the two loss terms. 
The training process of FMS are illustrated in \hyperref[alg:cfm]{Algorithm \ref{alg:cfm}}.

It is worth noticing that the initial context representation derived from $m^h_{j-1}(\boldsymbol{h}^*)$ and $m^t_{j-1}(\boldsymbol{t}^*)$ plays a dual role: it directly contributes to the base score for relation prediction, and it also serves as the condition \(\boldsymbol{z}\) for the flow-matching module, thereby influencing the dynamic modulation factor.

\begin{algorithm}[t]
\caption{FMS Training Process}
\label{alg:cfm}
\begin{algorithmic}[1]
\STATE \textbf{Input:} Batch data \((h, t)\) and \(r\), where \((h, t)\) are entity pairs, \(r\) are relation labels.
\STATE \textbf{Parameters:} Relation embeddings, Semantic Context Learning Module, MLP, Linear.
\smallskip
\REPEAT
    % \STATE \COMMENT{Compute context-aware embeddings}
    \STATE \(\textbf{\textit{h}}^* = \text{Semantic Context Learning Module}(h)\);
    \STATE \(\textbf{\textit{t}}^* = \text{Semantic Context Learning Module}(t)\);
    % \smallskip
    % \STATE \COMMENT{Apply Condition Flow-Matching}
    \STATE \(t, x_t, u_t \sim \text{Condition Flow-Matching Module}(\textbf{\textit{h}}^*, \textbf{\textit{t}}^*)\);
    \STATE \(v_{\theta}(t, \boldsymbol{x}) = \text{MLP}([x_t, t])\);
    \STATE \(s_{(h,t)} = \text{Linear}([\textbf{\textit{h}}^*, \textbf{\textit{t}}^*])\);
    \STATE \(s_{(h,t)} = s_{(h,t)} \odot v_{\theta}(t, \boldsymbol{\mathit{x}})\);
    % \smallskip
    % \STATE \COMMENT{Compute losses}
    \STATE \(\mathcal{L}_{cfm}(\theta) = \mathbb{E}_{t, q(\boldsymbol{\mathit{z}}), p_{t}(\boldsymbol{\mathit{x}} \mid \boldsymbol{\mathit{z}})} \left\| v_{\theta}(t, \boldsymbol{\mathit{x}}) - u_{t}(\boldsymbol{\mathit{x}}| \boldsymbol{\mathit{z}}) \right\|_{2}^{2}\);
    \STATE \(\mathcal{L}_{pred}(\Phi_S, \theta) = -\sum_{(h,r,t) \in \mathcal{D}} \log(p(r|h,t))\);
    \STATE \(\mathcal{L}(\Phi_S, \theta) = \mathcal{L}_{pred}(\Phi_S, \theta) + \lambda \mathcal{L}_{cfm}(\theta)\);
    % \smallskip
    % \STATE \COMMENT{Update parameters}
    \STATE Update model parameters using gradient descent on \(\mathcal{L}\).
\UNTIL{converged}
\end{algorithmic}
\end{algorithm}

To systematically validate the proposed Flow-Modulated Scoring (FMS) framework, our experimental evaluation is structured across two subsequent sections, each addressing a critical and complementary task within Knowledge Graph Completion (KGC).

\textcolor{red}{Section} \ref{sec4} is dedicated to Relation Prediction, the primary focus of this work. This task is not only pivotal for enabling fine-grained reasoning in knowledge graphs but also serves as the principal validation for the core design philosophy of FMS—its ability to capture the contextual dependency and dynamic evolutionary nature of relational semantics. The objective of this chapter is therefore to demonstrate the superior performance and technical advancements of FMS through a rigorous and comprehensive comparison against a wide array of state-of-the-art models.

Building upon this foundation, \textcolor{red}{Section} \ref{sec5} extends our assessment to the more fundamental and canonical task of Entity Prediction. While the FMS framework was conceived to address the profound complexities of relation modeling, its powerful representation learning and reasoning capabilities are hypothesized to possess broader applicability. Consequently, this chapter aims to verify the versatility and generalization power of the FMS framework. By adapting the model for entity prediction, we seek to establish that FMS is not merely a specialized solution for a single task but rather a versatile framework capable of empowering a spectrum of KGC challenges. This will comprehensively showcase its potential and practical utility as a novel paradigm for knowledge graph completion.

\section{Relation Prediction Experiment}
\label{sec4}
To evaluate the effectiveness of our FMS, we have designed a series of experiments to address the following research questions:
\begin{itemize}
    \item \textbf{RQ1:} How does the performance of FMS compare to a diverse range of mainstream prediction models?
    \item \textbf{RQ2:} What distinct contributions do the key components of FMS offer to the overall performance? Additionally, how does the model's performance adapt and respond to variations in hyperparameter settings?
    \item \textbf{RQ3:} How does FMS perform in terms of its number of parameters?
    \item \textbf{RQ4:} How does FMS modulate the static scores of triplets to capture rich contextual information shaping relation semantics?
    \item \textbf{RQ5:} How does FMS modulate the static scores of triplets to capture the dynamic evolution of relations?
\end{itemize}

\subsection{Transductive Experiment  Settings}
\label{sec4.1}
\subsubsection{Datasets}
We conduct experiments on six knowledge graph datasets:
(i) FB15K~\cite{toutanova-chen-2015-observed}, derived from Freebase, a large-scale KG of general human knowledge;
(ii) FB15K-237~\cite{toutanova-chen-2015-observed}, a subset of FB15K where inverse relations are removed;
(iii) WN18~\cite{dettmers2018convolutional}, containing conceptual-semantic and lexical relations among English words from WordNet;
(iv) WN18RR~\cite{dettmers2018convolutional}, a subset of WN18 where inverse relations are removed;
(v) NELL995~\cite{xiong2017deeppath}, extracted from the 995th iteration of the NELL system, containing general knowledge;
(vi) DDB14~\cite{wang2021relational}, collected from Disease Database, a medical database containing terminologies, concepts (diseases, symptoms, drugs), and their relationships.
The statistics above are summarized in \autoref{tab:all_dataset_stats_final}.

\begin{table}[t]
  \centering
    \caption{Statistics of all datasets. $\mathbb{E}[d]$ and $\mathrm{Var}[d]$ are mean and variance of the node degree distribution, respectively.}
    \vskip -0.1in
  \label{tab:all_dataset_stats_final} 
  \renewcommand{\arraystretch}{1.1} 
  \setlength{\tabcolsep}{3.5pt} 
  \begin{tabular}{c | c c c c c c c}
    \toprule
    \textbf{Datasets} & \textbf{Nodes} & \textbf{Relations} & \textbf{Train} & \textbf{Val} & \textbf{Test} & \textbf{$\mathbb{E}[d]$} & \textbf{$\mathrm{Var}[d]$} \\
    \midrule
    FB15K       & 15.0k & 1.3k & 483.1k & 50.0k & 59.1k & 64.6 & 32.4k \\
    FB15K-237   & 14.5k & 237  & 272.1k & 17.5k & 20.5k & 37.4 & 12.3k \\
    WN18        & 40.9k & 18   & 141.4k & 5.0k  & 5.0k  & 6.9  & 236.4 \\
    WN18RR      & 40.9k & 11   & 86.8k  & 3.0k  & 3.1k  & 4.2  & 64.3 \\
    NELL995     & 63.9k & 198  & 137.5k & 5.0k  & 5.0k  & 4.3  & 750.6 \\
    DDB14       & 9.2k  & 14   & 36.6k  & 4.0k  & 4.0k  & 7.9  & 978.8 \\
    \bottomrule
  \end{tabular}
\vskip -0.1in
\end{table}

\subsubsection{Baselines}
 Based on the concept of Knowledge Graph Completion, we categorize all baseline methods into four groups: 
 \textbf{(i) Emb-based:} TransE~\cite{bordes2013translating}, ComplEx~\cite{trouillon2016complex}, DistMult~\cite{yang2014embedding}, RotatE~\cite{sun2019rotate}, SimplE~\cite{kazemi2018simple}, QuatE~\cite{zhang2019quaternion}; 
 \textbf{(ii) Rule-based:} DRUM~\cite{sadeghian2019drum}, and PTransE~\cite{peng2022path}; 
 \textbf{(iii) GNN-Based:} 
 R-GCN~\cite{schlichtkrull2018modeling}, Pathcon~\cite{wang2021relational}, LASS~\cite{shen2022joint} and CBLiP~\cite{dutta2025replacing}; 
\textbf{(iv) LLM-Based:} KG-BERT~\cite{yao2019kg}, KGE-BERT~\cite{nadkarni2021scientific}, and GilBERT~\cite{nassiri2022knowledge}

\subsubsection{Evaluation Metrics}
To ensure a fair comparison, our training and evaluation setup follows that of PathCon~\cite{wang2021relational}\footnote{ \url{https://github.com/hyren/PathCon}}. 
For each test triplet $(h, r, t)$, we formulate a query $(h, ?, t)$\footnote{Some related work formulates this as predicting missing tail (or head) entities given a head (or tail) entity and a relation. These two problems (entity prediction and relation prediction) are reducible to each other. For instance, a relation prediction model $\Phi(\cdot | h, t)$ can be converted to a tail prediction model $\Gamma(\cdot | h, r) = \operatorname{SoftMax}_t (\Phi(r | h, t))$, and vice versa. Given this equivalence, this work focuses on relation prediction. We also discuss the entity prediction task in \autoref{sec5}.} where the model's objective is to predict the correct relation $r$. 
The model is required to rank the ground-truth relation $r$ against all other relations in the knowledge graph, which serve as negative samples. 
For instance, a negative candidate would be $(h, r', t)$, where $r' \neq r$. 
Analogous to link prediction, we report Mean Reciprocal Rank (MRR) and Hits@N under the filtered setting, which removes any other known true relations between the entity pair $(h, t)$ from the candidate list before ranking. 
Higher scores for these metrics indicate better model performance.

\begin{table}[t] 
    \centering
      \renewcommand{\arraystretch}{1.2} 
  \setlength{\tabcolsep}{15pt} 
    \caption{Best hyper-parameter configurations for FMS. The order for dataset-specific values is: FB15K, FB15K-237, WN18, WN18RR, NL995, DDB14.}
    \vskip -0.1in
    \label{tab:hyperparams}
    \footnotesize 
    \begin{tabular}{@{}lc@{}} 
    \toprule
    \textbf{Hyper-parameter} & \textbf{Value / Setting} \\
    \midrule
    $\lambda$               &1.2 \\
    Epoch                   & 20 \\
    Dimension (dim)         & 64 \\
    L2 Regularization       & $1 \times 10^{-7}$ \\
    Learning Rate           & $5 \times 10^{-3}$ \\
    Neighbor Aggregation    & Attention \\
    Temperature             & 0.95 \\
    \midrule
    \multicolumn{2}{@{}l}{\textit{Dataset-specific Parameters}} \\
    Batch Size              & 128 / 128 / 128 / 128 / 128 / 64 \\
    Context Hops            & 2 / 2 / 3 / 3 / 2 / 2 \\
    Neighbor Samples        & 16 / 16 / 8 / 8 / 8 / 8 \\
    Num. Top-K              & 10 / 10 / 4 / 4 / 3 / 3 \\
    Num. Attention Heads    & 4 / 4 / 4 / 4 / 4 / 4 \\
    \bottomrule
    \end{tabular}
\vskip -0.1in
\end{table}

\subsubsection{Implementation Details}

All experiments are conducted on two NVIDIA 4090 GPUs. Our investigations revealed that FMS's performance is highly sensitive to certain hyperparameters, particularly the number of context hops and the Top-K value, which depend on dataset characteristics. Consequently, to achieve optimal performance, we conducted a comprehensive hyperparameter search independently for each dataset across the following pre-defined spaces:
\begin{itemize}
    \item Dimension of hidden states: \{16, 32, 64, 128\};
    \item Weight of L2 loss term: \{$10^{-8}$, $10^{-7}$, $10^{-6}$, $10^{-5}$\};
    \item Learning rate: \{0.0005, 0.005, 0.01, 0.05, 0.1\};
    \item Number of context hops: \{1, 2, 3, 4\};
    \item Number of neighbor samples: \{8, 16, 24, 32\};
    \item Number of Top-K: \{3, 4, 5, 8, 10\}.
\end{itemize}
The best-performing configurations derived from this search are detailed in \autoref{tab:hyperparams}. For result stability, each experiment was repeated three times, and we report the average performance.  Our implementation is open-sourced and available on GitHub to support reproducibility.\footnote{\url{https://github.com/yuanwuyuan9/FMS}} 

\subsection{Inductive Experiment Settings}
\label{sec4.2}
\subsubsection{Datasets}
Following ~\cite{zhang2022knowledge}, we utilize a total of twelve subsets,
encompassing four distinct versions each, which are derived from
established datasets: WN18RR, FB15k237, and NELL-995. Each subset is uniquely characterized by a disparate split between the training and test sets, ensuring a robust and comprehensive evaluation schema. For a detailed discourse on the specific splits and a nuanced
presentation of the statistical attributes of each subset, one can refer
to the comprehensive descriptions available in ~\cite{zhang2022knowledge}.

\subsubsection{Baselines}
For inductive learning, traditional Embedding-based methods are not suitable for the task due to their inability to generalize to unseen entities, we choose the following two classes of classical and effective models: 
\textbf{(i) Rule-based:} RuleN~\cite{meilicke2018fine}, NeuralLP~\cite{yang2017differentiable}, DRUM~\cite{sadeghian2019drum};
\textbf{(ii) GNN-based:} GraIL~\cite{teru2020inductive}, PathCon~\cite{wang2021relational}.

\subsubsection{Evaluation Metrics}
The setup and parameter selection are consistent with the transduction experiments.

\begin{table*}[t]
  \centering
      \caption{Transductive performance comparison on all six datasets. The results here are expressed as percentages. The best results are shown in bold while the second-best results are shown in the underline.}
      \vskip -0.1in
  \label{tab:main_results_final_formatted}
  \setlength{\tabcolsep}{2.5pt} % Further adjusted for potential width with bold/underline
  \renewcommand{\arraystretch}{1.1}
  \begin{tabular}{c|c|ccc|ccc|ccc|ccc|ccc|ccc}
    \toprule
    \multirow{2}{*}{\textbf{Type}} & \multirow{2}{*}{\textbf{Model}} & \multicolumn{3}{c|}{\textbf{FB15K}} & \multicolumn{3}{c|}{\textbf{FB15K-237}} & \multicolumn{3}{c|}{\textbf{WN18}} & \multicolumn{3}{c|}{\textbf{WN18RR}} & \multicolumn{3}{c|}{\textbf{NELL995}} & \multicolumn{3}{c}{\textbf{DDB14}}\\
    \cmidrule(lr){3-5} \cmidrule(lr){6-8} \cmidrule(lr){9-11} \cmidrule(lr){12-14} \cmidrule(lr){15-17} \cmidrule(lr){18-20}
                                & & \textbf{MRR} & \textbf{H@1} & \textbf{H@3} & \textbf{MRR} & \textbf{H@1} & \textbf{H@3} & \textbf{MRR} & \textbf{H@1} & \textbf{H@3} & \textbf{MRR} & \textbf{H@1} & \textbf{H@3} & \textbf{MRR} & \textbf{H@1} & \textbf{H@3} & \textbf{MRR} & \textbf{H@1} & \textbf{H@3} \\
    \midrule
    \multicolumn{1}{c|}{\multirow{6}{*}{\centering \textbf{Emb-based}}} & TransE                 & 96.2 & 94.0 & 98.2 & 96.6 & 94.6 & 98.4 & 97.1 & 95.5 & 98.4 & 78.4 & 66.9 & 87.0 & 84.1 & 78.1 & 88.9 & 96.6 & 94.8 & 98.0 \\
     & ComplEx                   & 90.1 & 84.4 & 95.2 & 92.4 & 87.9 & 97.0 & 98.5 & 97.9 & 99.1 & 84.0 & 77.7 & 88.0 & 70.3 & 62.5 & 76.5 & 95.3 & 93.1 & 96.8 \\
     & DistMult                  & 66.1 & 43.9 & 86.8 & 87.5 & 80.6 & 93.6 & 78.6 & 58.4 & 98.7 & 84.7 & 78.7 & 89.1 & 63.4 & 52.4 & 72.0 & 92.7 & 88.6 & 96.1 \\
     & RotatE                  & 97.9 & 96.7 & 98.6 & 97.0 & 95.1 & 98.0 & 98.4 & 97.9 & 98.6 & 79.9 & 73.5 & 82.3 & 72.9 & 69.1 & 75.6 & 95.3 & 93.4 & 96.4 \\
     & SimplE                  & 98.3 & 97.2 & 99.1 & 97.1 & 95.5 & 98.7 & 97.2 & 96.4 & 97.6 & 73.0 & 65.9 & 75.5 & 71.6 & 67.1 & 74.8 & 92.4 & 89.2 & 94.8 \\
     & QuatE                   & 98.3 & 97.2 & 99.1 & 97.4 & 95.8 & 98.8 & 98.1 & 97.5 & 98.3 & 82.3 & 76.7 & 85.2 & 75.2 & 70.6 & 78.3 & 94.6 & 92.2 & 96.2 \\
    \midrule
    \multicolumn{1}{c|}{\multirow{2}{*}{\centering \textbf{Rule-based}}}
    & DRUM                   & 94.5 & 94.5 & 97.8 & 95.9 & 90.5 & 95.8 & 96.9 & 95.6 & 98.0 & 85.4 & 77.8 & 91.2 & 71.5 & 64.0 & 74.0 & 95.8 & 93.0 & 98.7 \\
   & PTransE            & - & - & - & 96.2 & 96.0 & 98.5 & 98.7 & 98.2 & 98.9 & 85.3 & 87.5 & 94.7 & - & - & - & - & - & - \\
     \midrule
    \multicolumn{1}{c|}{\multirow{4}{*}{\centering \textbf{GNN-based}}} 
    & R-GCN                  & - & - & - & 93.1 & 90.3 & 94.2 & 90.9 & 82.4 & 97.9 & 82.2 & 75.0 & 84.8 & - & - & - & - & - & - \\
   & Pathcon                & \underline{98.4} & \underline{97.4} & \underline{99.5} & \underline{97.9} & \underline{96.4} & \underline{99.4} & \underline{99.3} & \underline{98.8} & \underline{99.8} & 97.4 & 95.4 & \underline{99.4} & 89.6 & 84.4 & 94.1 & \underline{98.0} & \underline{96.6} & \underline{99.5} \\
   & LASS                   & - & - & - & 97.3 & 95.5 & 98.8 & 96.7 & 95.9 & 98.0 & 95.9 & 94.2 & 96.4 & - & - & - & - & - & - \\
   & CBLiP                  & - & - & - & 97.1 & 94.9 & 99.2 & - & - & - & \underline{97.6} & \underline{96.0} & 99.3 & \underline{91.9} & \underline{86.8} & \underline{96.4} & - & - & - \\
    \midrule
    \multicolumn{1}{c|}{\multirow{3}{*}{\centering \textbf{LLM-based}}}
    & KG-BERT                & - & - & - & 97.3 & 95.2 & 98.6 & 96.7 & 94.6 & 98.2 & 94.2 & 91.7 & 96.8 & - & - & - & - & - & - \\
   & KGE-BERT               & - & - & - & 97.2 & 94.3 & 99.0 & 96.5 & 95.7 & 97.9 & 94.8 & 92.4 & 96.1 & - & - & - & - & - & - \\
   & GilBERT                & - & - & - & 95.1 & 91.4 & 97.0 & 96.5 & 94.7 & 98.1 & 94.1 & 91.6 & 96.7 & - & - & - & - & - & - \\
    \midrule
    \multicolumn{1}{c|}{\centering \textbf{Ours}} & \textbf{FMS}  & \textbf{99.6} & \textbf{99.5} & \textbf{99.8} & \textbf{99.8} & \textbf{99.7} & \textbf{99.9} & \textbf{99.7} & \textbf{99.5} & \textbf{99.9} & \textbf{99.9} & \textbf{99.8} & \textbf{99.9} & \textbf{99.1} & \textbf{98.6} & \textbf{99.5} & \textbf{99.8} & \textbf{99.7} & \textbf{99.9} \\
    \bottomrule
  \end{tabular}
\end{table*}

\begin{table*}[t]
\vskip -0.1in
  \centering
    \caption{Inductive results (Hits@10). The best results are shown in bold while the second-best results are shown underlined.}
    \vskip -0.1in
  \label{tab:inductive_hits10_formatted_rounded}
  \setlength{\tabcolsep}{9pt}
\renewcommand{\arraystretch}{1.1}
  \begin{tabular}{c|c|cccc|cccc|cccc}
    \toprule
    \multirow{2}{*}{\textbf{Type}} & \multirow{2}{*}{\textbf{Model}} & \multicolumn{4}{c|}{\textbf{WN18RR}} & \multicolumn{4}{c|}{\textbf{FB15k-237}} & \multicolumn{4}{c}{\textbf{NELL-995}} \\
    \cmidrule(lr){3-6} \cmidrule(lr){7-10} \cmidrule(lr){11-14}
                                & & \textbf{v1} & \textbf{v2} & \textbf{v3} & \textbf{v4} & \textbf{v1} & \textbf{v2} & \textbf{v3} & \textbf{v4} & \textbf{v1} & \textbf{v2} & \textbf{v3} & \textbf{v4} \\
    \midrule
    \multicolumn{1}{c|}{\multirow{3}{*}{\centering \textbf{Rule-based}}} 
    & Neural-LP & 74.4 & 68.9 & 46.2 & 67.1 & 52.9 & 58.9 & 52.9 & 55.9 & 40.8 & 78.7 & 82.7 & 80.6 \\
    & DRUM      & 74.4 & 68.9 & 46.2 & 67.1 & 52.9 & 58.7 & 52.9 & 55.9 & 19.4 & 78.6 & 82.7 & 80.6 \\
    & RuleN     & 80.9 & 78.2 & 53.4 & 71.6 & 49.8 & 77.8 & 87.7 & 85.6 & \underline{53.5} & 81.8 & 77.3 & 61.4 \\
    \midrule
    \multicolumn{1}{c|}{\multirow{2}{*}{\centering \textbf{GNN-based}}} 
    & GraIL     & 82.5 & 78.7 & 58.4 & 73.4 & 64.2 & 81.8 & 82.8 & 89.3 & \textbf{59.5} & 93.3 & 91.4 & 73.2 \\
    & PathCon   & \underline{97.6} & \underline{97.1} & \underline{97.4} & \underline{98.7} & \underline{94.1} & \underline{95.6} & \underline{95.9} & \underline{96.0} & - & \underline{95.8} & \underline{94.5} & \underline{95.6} \\
    \midrule
    \multicolumn{1}{c|}{\multirow{1}{*}{\centering \textbf{Ours}}} 
    & FMS       & \textbf{99.2} & \textbf{98.4} & \textbf{99.0} & \textbf{99.5} & \textbf{98.4} & \textbf{99.5} & \textbf{99.6} & \textbf{99.9} & - & \textbf{99.5} & \textbf{99.2} & \textbf{99.8} \\
    \bottomrule
  \end{tabular}
%\vskip -0.1in
\end{table*}

\begin{table}[t]
  \centering
  \caption{Results by relation category on FB15k-237.}
  \vskip -0.1in
  \label{tab:1_N}
    \setlength{\tabcolsep}{13pt}
     \renewcommand{\arraystretch}{1.1}
  \begin{tabular}{c|cc|cc}
    \toprule
    \multicolumn{1}{c|}{\raisebox{-\totalheight}{\textbf{Model}}} & \multicolumn{2}{c|}{\textbf{1-1}} & \multicolumn{2}{c}{\textbf{1-N}} \\
    \cmidrule(lr){2-3} \cmidrule(lr){4-5}
                               & \textbf{MRR} & \textbf{Hit@1} & \textbf{MRR} & \textbf{Hit@1} \\
    \midrule
    PathCon               &92.4  &88.9  &80.9  &74.2  \\
        FMS                         &\textbf{99.8}  &\textbf{99.7}  &\textbf{99.6}  &\textbf{99.3}  \\
    \bottomrule
  \end{tabular}
  \vskip -0.1in
\end{table}

\begin{table}[t]
  \centering
  \caption{Ablation study on key components of FMS.}
  \vskip -0.1in
\label{tab:ablation_fms}
 \renewcommand{\arraystretch}{1.1}
     \setlength{\tabcolsep}{7pt}
  \begin{tabular}{l|cc|cc}
    \toprule
    \multicolumn{1}{c|}{\raisebox{-\totalheight}{\textbf{Ablation Settings}}} & \multicolumn{2}{c|}{\textbf{FB15k-237}} & \multicolumn{2}{c}{\textbf{WN18RR}} \\
    \cmidrule(lr){2-3} \cmidrule(lr){4-5}
                               & \textbf{MRR} & \textbf{Hit@1} & \textbf{MRR} & \textbf{Hit@1} \\
    \midrule
    FMS                         &\textbf{99.8}  &\textbf{99.7}  &\textbf{99.9}  &\textbf{99.8}  \\
    FMS w/o Top-K               &97.8  &96.1  &94.3  &89.4  \\
    FMS w/o Score               &98.2  &97.1  &97.4  &95.9  \\
    FMS w/o Flow-Matching       &98.7  &97.9  &98.4  &97.5  \\
    \bottomrule
  \end{tabular}
  \vskip -0.1in
\end{table}

\subsection{Overall Performance Comparison: RQ1}
To answer RQ1, we conducted extensive experiments comparing FMS with the strongest baseline we know of, both in terms of transductive and inductive relation prediction.

The transductive results are summarized in \autoref{tab:main_results_final_formatted}. FMS demonstrates remarkable and consistent improvements across all six datasets. 
Specifically, FMS achieves a \textbf{1.2\% relative}, \textbf{1.7\% relative}, \textbf{0.4\% relative}, \textbf{2.5\% relative}, \textbf{10.6\% relative}, and \textbf{1.8\% relative} increase in MRR scores over the best baselines on FB15K, FB15k-237, WN18, WN18RR, NELL995, and DDB14, respectively. These consistent and significant gains across diverse datasets, particularly the substantial improvement on NELL995, underscore the robustness and superior capability of our proposed FMS model.
Furthermore, it is important to note that, compared to all kinds of models, FMS achieves better results for MRR scores on all datasets, establishing its state-of-the-art performance.

To further analyze the performance of FMS, we decomposed FB15k-237 based on relation categories similar to those defined in~\cite{wang2014knowledge}. We redefined the 1-N task by specifically isolating a test subset where each entity pair has two or more relation types. As shown in~\autoref{tab:1_N}, FMS achieved overwhelmingly superior performance, indicating its successful capture of the dynamic nature of relations.

The inductive results are summarized in \autoref{tab:inductive_hits10_formatted_rounded}. FMS shows remarkable improvement across Hits@10 metrics on twelve datasets. It is noteworthy that while the optimal performance of other models typically ranges between 70\% and 90\% except PathCon, the FMS model almost consistently achieves between 98\% and 99.9\%, demonstrating its overwhelming advantage in inductive learning capabilities. Furthermore, FMS maintains near-perfect performance across different versions and splits of various datasets, indicating that the model possesses excellent stability and robustness to data variations.

\subsection{Ablation Study: RQ2}
\label{sec4.4}

\subsubsection{Key Modules Ablation Study}
To verify the effectiveness of our proposed modules, we develop three distinct model variants:
\begin{itemize}
    \item \textbf{w/o Top-K}: We replace the Top-K selection mechanism based on semantic similarity in the Semantic Context Learning Module with a simple random sampling strategy.
    \item \textbf{w/o Score}: We remove the energy-model inspired scoring function, and instead use dot product similarity for scoring. 
    \item \textbf{w/o Flow-Matching}: We remove the Condition Flow-Match-ing Module, relying solely on the basic static scores of entity pairs.
\end{itemize}

The ablation study results are presented in \autoref{tab:ablation_fms}, leading to the following key conclusions: 
The Top-K selection mechanism exerts the most significant impact on the overall FMS performance, with its absence leading to a substantial decline in performance, particularly on the WN18RR dataset. This underscores Top-K's pivotal role in filtering critical information and suppressing noise. Secondly, the energy-model inspired scoring function is also crucial for evaluating and ranking candidate answers, and its removal results in a noticeable loss in performance. While each component is crucial, their collective impact exceeds the sum of their parts. This highlights that the synergistic interaction of semantic context selection, energy-based scoring, and dynamic flow modulation is the key factor enabling the FMS model to maintain its state-of-the-art performance.

%%================================%%
%%                         TODO：修改超参数分析的图的X、Y轴字体大小                         %%
%%================================%%

\begin{figure*}[htbp]
    \centering

    % Subfigure (a)
    \subfloat[\footnotesize Results with different Top-K.\label{fig:sub_topk_analysis}]{%
        \includegraphics[width=0.3\linewidth]{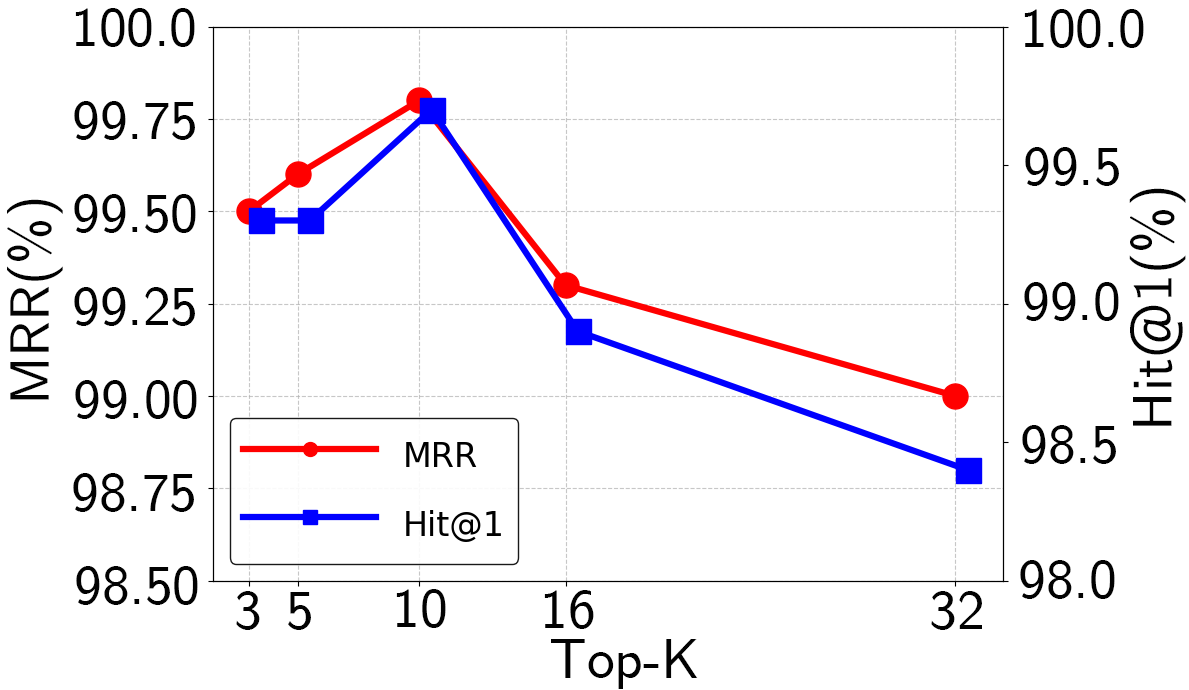}%
    }\hfill
    % Subfigure (b)
    \subfloat[\footnotesize Results with different aggregators.\label{fig:sub_aggregator_analysis}]{%
        \includegraphics[width=0.35\linewidth]{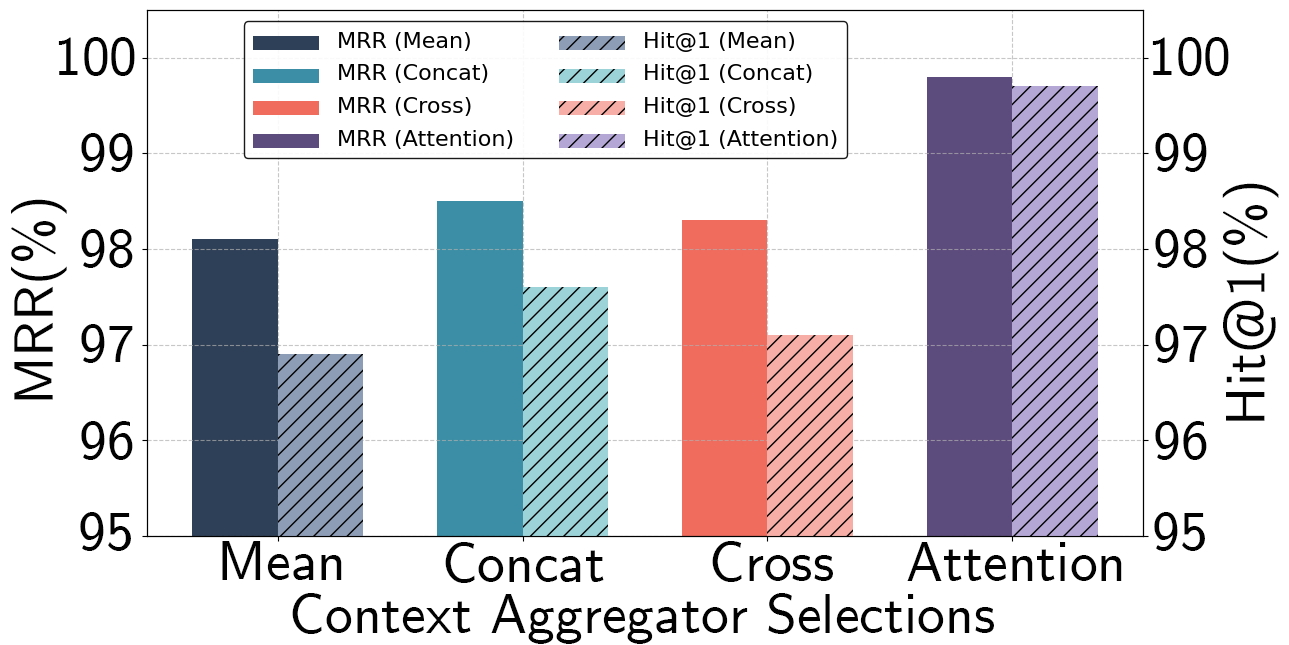}%
    }\hfill
    % Subfigure (c)
    \subfloat[\footnotesize Results with different hops.\label{fig:sub_hops_analysis}]{%
        \includegraphics[width=0.32\linewidth]{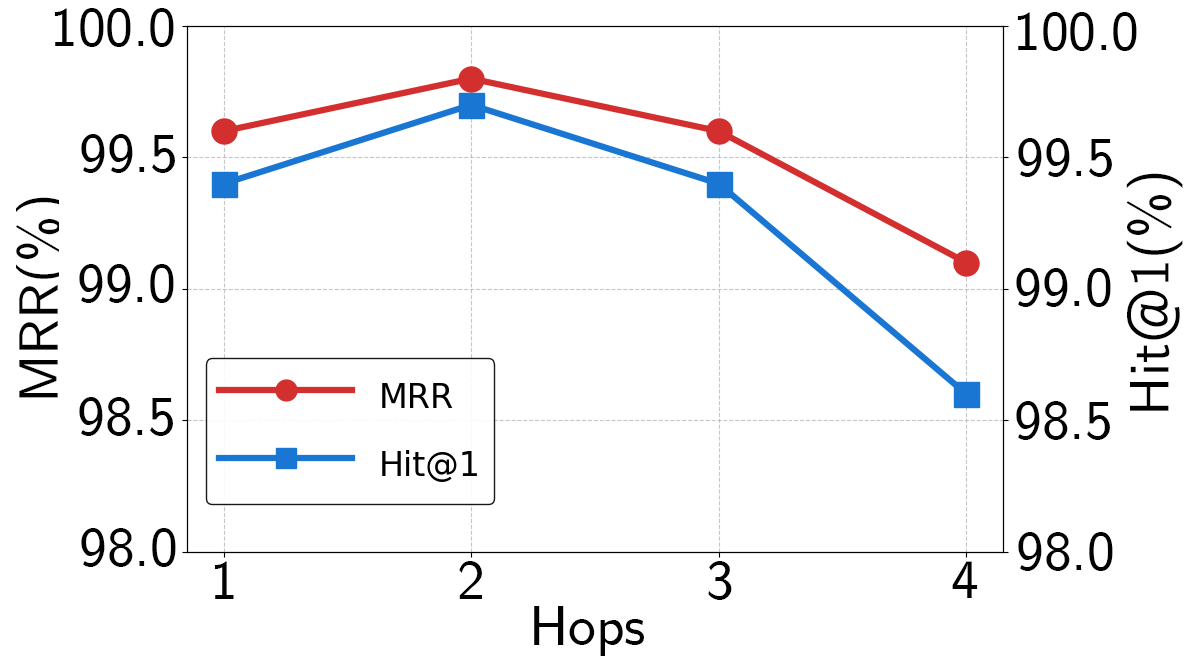}%
    }

    \caption{Hyper-parameters analysis of FMS on FB15k-237.}
    \label{fig:FMS_hyperparameter_analysis}
\vskip -0.2in
\end{figure*}

\subsubsection{Sensitivity to Key Hyperparameters}
In this study, we analyze the impact of key hyperparameters within the Semantic Context Learning Module and the effect of different aggregator choices. The results on the FB15k-237 dataset are presented in \autoref{fig:FMS_hyperparameter_analysis}.

\autoref{fig:FMS_hyperparameter_analysis}\subref{fig:sub_topk_analysis} shows how FMS performance varies with different Top-K values. The results demonstrate that a moderate K value (specifically 10 in this case) is critical. An overly small K provides insufficient context to leverage, while an excessively large K introduces increasing noise, potentially overwhelming the model with excessive contextual information and causing a drop in performance.

\autoref{fig:FMS_hyperparameter_analysis}\subref{fig:sub_aggregator_analysis} then shows the performance of FMS's Semantic Context Learning Module when using different context aggregators. It can be seen that our designed self-attention aggregator performs best, while the mean aggregator performs worst.

\autoref{fig:FMS_hyperparameter_analysis}\subref{fig:sub_hops_analysis} illustrates the selection of hop count on the graph. The results indicate that FMS achieves competitive performance at the first hop and optimal performance at the second hop. Further increasing the hop count, however, leads to performance degradation. This phenomenon is likely attributable to the exponential increase in neighbor sampling associated with deeper hops, which consequently introduces more noise. Nevertheless, owing to the mechanisms of Top-k selection and attention aggregation, the model maintains its state-of-the-art performance.

\begin{table}
  \centering
      \caption{Number of parameters of all models on DDB14.}
      \vskip -0.1in
  \label{tab:parameters_ddb14}
  {
    \setlength{\tabcolsep}{3pt}
  \begin{tabular}{c|cccccccc}
    \toprule
    \textbf{Model}   & \textbf{TransE} & \textbf{DisMult} &\textbf{RotatE}  & \textbf{QuatE} & \textbf{PathCon} & \textbf{FMS}  \\
    \midrule
    Param & 3.7M   & 3.7M    & 7.4M   & 14.7M & 0.06M   & \textbf{0.02M} \\ 
    \bottomrule
  \end{tabular}
  } % End the group
\vskip -0.1in
\end{table}

\begin{table}
  \centering
  \caption{Number of parameters of all models on FB15k-237.}
  \vskip -0.1in
  \label{tab:parameters_fb15k237}
  {
    \setlength{\tabcolsep}{3pt}
  \begin{tabular}{c|cccccc} % 
    \toprule
    \textbf{Model}   & \textbf{TransE} & \textbf{DisMult} &\textbf{RotatE}  & \textbf{QuatE} & \textbf{PathCon} & \textbf{FMS}  \\
    \midrule
    Param & 5.9M   & 5.9M    & 11.7M   & 23.6M & 1.67M   & \textbf{0.35M }\\ 
    \bottomrule
  \end{tabular}
  } % End the group
\vskip -0.1in
\end{table}

\begin{figure*}[t]
    \centering
    \subfloat[\footnotesize DDB14 (14 relation types)\label{fig:case-ddb14}]{
        \includegraphics[width=0.48\textwidth]{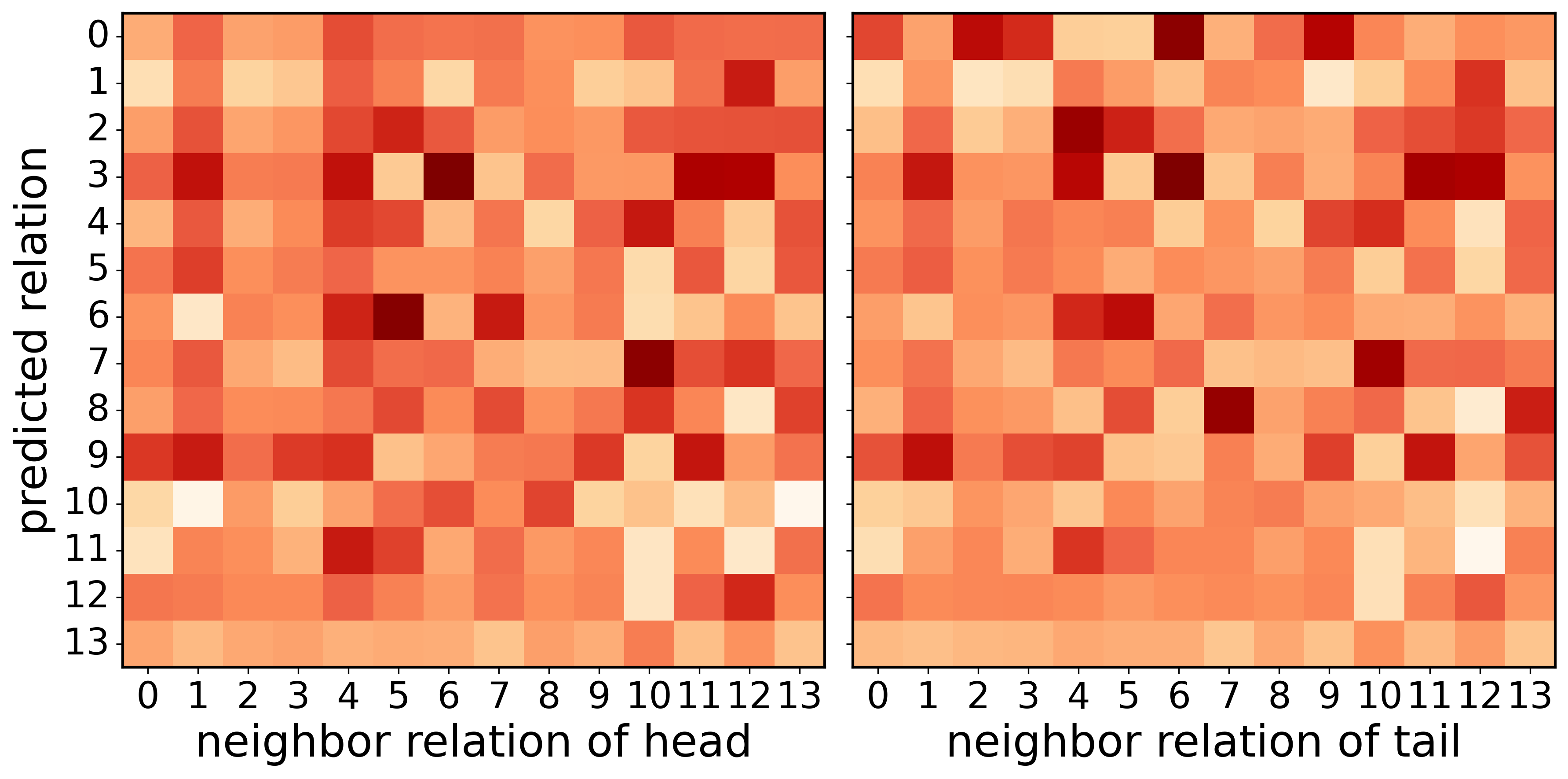}
    }
    \hfill
    \subfloat[\footnotesize WN18RR (11 relation types)\label{fig:case-wn18rr}]{
        \includegraphics[width=0.48\textwidth]{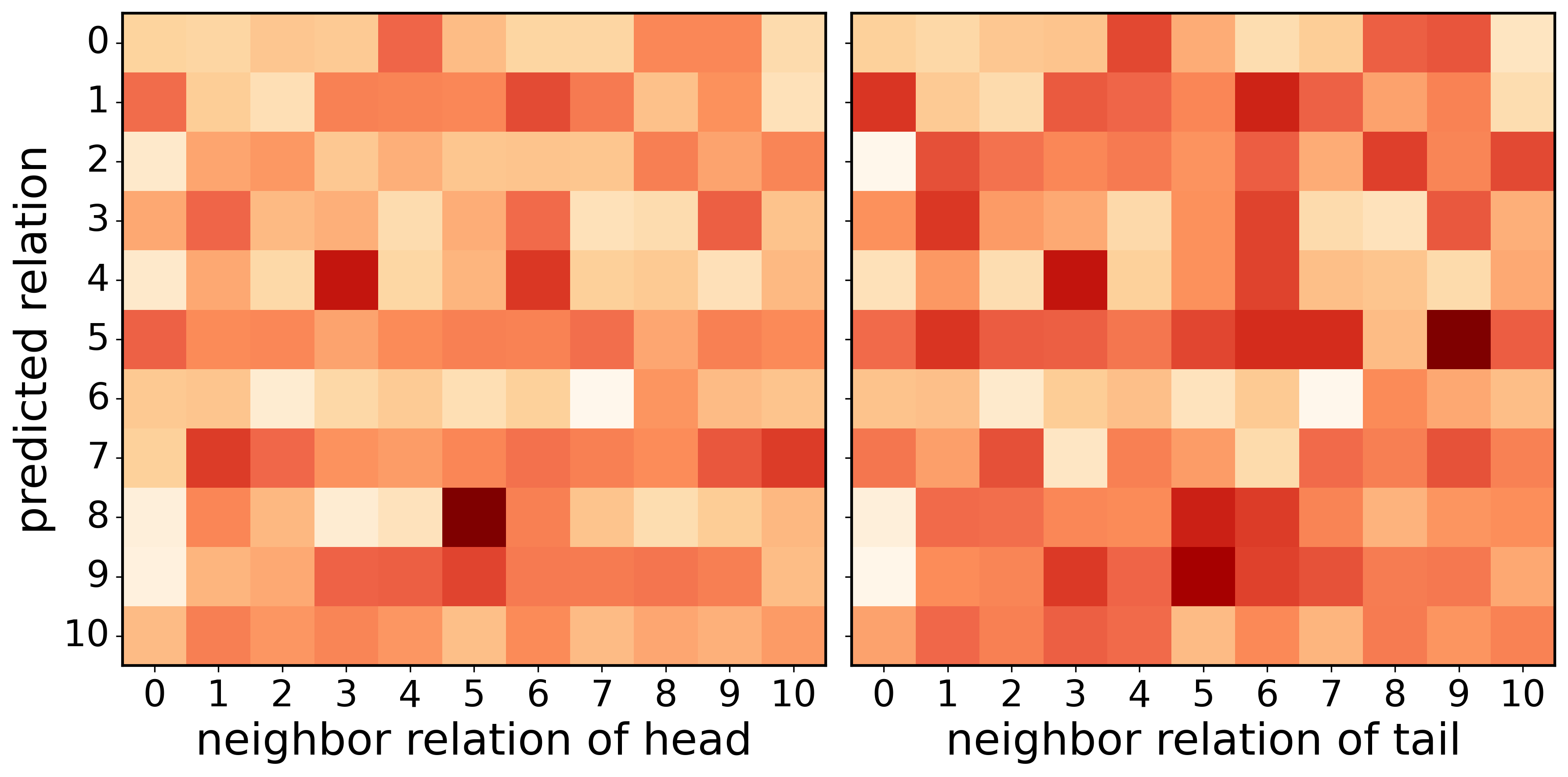}
    }
    \caption{The learned correlation between the contextual relations of head/tail entities and the predicted relations. Case studies are shown for the DDB14 dataset (a) and the WN18RR dataset (b).}
    \label{fig:correlation_case_study}
\vskip -0.1in
\end{figure*}

\subsection{Efficiency Analysis: RQ3}
We performed parametric volume analysis on both datasets, on the smaller dataset DDB14 as in \autoref{tab:parameters_ddb14} and on the larger dataset FB15k-237 as in \autoref{tab:parameters_fb15k237}. The model complexity is $\mathcal{O}(n \cdot K^{Hops})$.
The result demonstrates that FMS is much more
storage-efficient than embedding-based methods, since it does not
need to calculate entity embeddings. Furthermore, FMS is significantly more effective than GNN-based methods such as PathCon.
FMS can achieve superior performance with fewer hops and a smaller number of sampled neighbors.
Additionally, it does not require explicitly modeling paths from the head entity to the tail entity like PathCon;
instead, it employs flow-matching to learn a dynamic evolution vector field,
which represents semantically context-guided paths.

% We report a comparison of the training curves for FMS and PathCon on the DDB14 dataset in \autoref{fig:enter-label}.
% It is evident that FMS achieves exceptional performance within a remarkably short period,
% and its training efficiency is significantly greater than that of PathCon.

% \begin{figure}[h]
%     \centering
% \includegraphics[width=0.96\linewidth]{efficiency.png}
%     \caption{Training efficiency of FMS and PathCon on DDB14 with close settings.}
%     \label{fig:enter-label}
% \end{figure}

%%================================%%
%%                         TODO：上下文与预测关系的相关性图                         %%
%%================================%%

\subsection{Case Study on Contextual Information: RQ4}

To demonstrate the robust capability of our FMS in capturing the semantic context of relations, we present heatmaps of the modulated context-aggregated head/tail entity matrix on DDB14 (14 relation types) and WN18RR (11 relation types) in~\autoref{fig:correlation_case_study}. The intensity of the entries in~\autoref{fig:correlation_case_study} represents the strength of the correlation between the presence of a contextual relation and a predicted relation. Relation IDs and their corresponding meanings for DDB14 and WN18RR are detailed in~\autoref{tab:relation_ids}.

\begin{table}[t]
\centering
\caption{Relation IDs and Meanings on DDB14 and WN18RR Datasets.}
\vskip -0.1in
\label{tab:relation_ids}
\setlength{\tabcolsep}{3pt}
\begin{tabular}{llll}
\toprule
\multicolumn{4}{c}{\textbf{DDB14 Dataset}} \\
\midrule
ID & Relation Meaning & ID & Relation Meaning \\
\midrule
0 & belong(s) to the category of & 7 & interacts with \\
1 & is a category subset of & 8 & belongs to the drug family of \\
2 & may cause & 9 & belongs to drug super-family \\
3 & is a subtype of & 10 & is a vector for \\
4 & is a risk factor for & 11 & may be allelic with \\
5 & is associated with & 12 & see also \\
6 & may contraindicate & 13 & is an ingredient of \\
\midrule \midrule
\multicolumn{4}{c}{\textbf{WN18RR Dataset}} \\
\midrule
ID & Relation Meaning & ID & Relation Meaning \\
\midrule
0 & hypernym & 6 & has\_part \\
1 & derivationally\_related\_form & 7 & member\_of\_domain\_usage \\
2 & instance\_hypernym & 8 & member\_of\_domain\_region \\
3 & also\_see & 9 & verb\_group \\
4 & member\_meronym & 10 & similar\_to \\
5 & synset\_domain\_topic\_of & & \\
\bottomrule
\end{tabular}
\vskip -0.15in
\end{table}

~\autoref{fig:correlation_case_study} reveals a sophisticated learning mechanism that moves beyond superficial pattern matching. The prevalence of strong off-diagonal correlations indicates that the model transcends simple pattern memorization---typified by rules like $(h, R, t') \Rightarrow (h, R, t)$ ~\cite{wang2021relational}. Instead, the model constructs a network of latent semantic rules, which we categorize into the following core principles.

\begin{itemize}
    \item \textbf{Domain-Specific Rule Induction:} FMS demonstrates a remarkable ability to induce logical rules that are specifically adapted to the intrinsic properties of a given knowledge domain.
    
    In DDB14, it captures a clinical cause-and-effect relationship, inferring a rule that connects risk factors to clinical contraindications:
    \begin{equation*}
        \small
        (h, \text{is\_risk\_factor\_for}, t') \Rightarrow (h, \text{may\_contraindicate}, t).
    \end{equation*}

    In WN18RR, it models a fundamental linguistic principle linking etymology to semantics:
    \begin{equation*}
        \small
        (h, \text{etymologically\_related}, t') \Rightarrow (h, \text{usage\_domain}, t).
    \end{equation*}

    \item \textbf{Modeling of Hierarchical Structures:} FMS effectively captures hierarchical knowledge structures, enabling it to perform complex, multi-level reasoning from specific instances to general categories.
   
    In DDB14, it masters compositional reasoning, inferring macroscopic class membership from microscopic constituents:
    \begin{equation*}
        \small
        (h, \text{is\_ingredient\_of}, t') \Rightarrow (h, \text{belongs\_to\_drug\_family}, t).
    \end{equation*}
    In WN18RR, it demonstrates a form of inductive inference, leveraging horizontal similarity to determine vertical class membership:
    \begin{equation*}
        \small
        \begin{split}
        (h, \text{similar\_to}, h') \wedge (h', \text{instance\_hypernym}, t) \\
        \Rightarrow (h, \text{instance\_hypernym}, t).
        \end{split}
    \end{equation*}
\end{itemize}

%%================================%%
%%                         TODO：这个Case感觉不太恰当，有重找一个Case分析的打算                        %%
%%================================%%

\subsection{Case Study on Dynamic Evolution: RQ5}

To quantitatively evaluate the impact of the Flow-Modulated Scoring (FMS) framework, we investigate a case study from the FB15k-237 dataset. The task, detailed in \autoref{tab:case_study_example}, involves predicting the multiple valid financial relationships between an arts college and a currency. 

\autoref{fig:tsne_fms} presents T-SNE~\cite{vanDerMaaten2008visualizing} visualizations to compare the representations learned by FMS against those from traditional static scoring methods.

Central to the figure, the "Static Score Embedding" (black square) represents the generalized initial understanding that a traditional static model might offer for the relation between "California College of the Arts" and "United States Dollar". Due to its singular nature, this single point struggles to simultaneously and precisely encapsulate the specific role of "United States Dollar" as a currency across six different financial contexts.
The power of FMS is evident in its subsequent dynamic processing. Observing the "Dynamic Scores" (colored stars) in the figure, we see that FMS is not content with this single, potentially ambiguous static representation. Instead, it actively flows from this static starting point to generate specific embedding representations for each potential relation. Each arrow in the figure symbolizes this context-driven modulation process, guiding the initial representation towards more relevant and fine-grained semantic space regions associated with specific relation types.

FMS successfully distinguishes between multiple concurrently valid relations. It does not merely identify the generic concept of "currency"; rather, it can precisely differentiate specific relations such as "Revenue Currency" and "Assets Currency" within the embedding space. This clearly demonstrates its exceptional performance in capturing the dynamic nature of knowledge graph relations.

\begin{table}[t]
    \centering
    \caption{Case study example from the FB15k-237 dataset.}
    \vskip -0.1in
    \label{tab:case_study_example}
    \renewcommand{\arraystretch}{1} 
    \begin{tabularx}{\linewidth}{l X}
        \toprule
        \textbf{Component} & \textbf{Value} \\
        \midrule
        Head Entity & California College of the Arts \\
        Tail Entity & United States Dollar \\
        \textbf{\textit{Query}} & \textit{\textbf{(California College of the Arts, ?, United States Dollar)}} \\
        \textbf{Ground-Truth} & 
                \begin{itemize}[leftmargin=*, topsep=0pt, partopsep=0pt, itemsep=0pt, parsep=0pt]
            \item Revenue Currency
            \item Operating Income Currency
            \item Local Tuition Currency
            \item Assets Currency
            \item Domestic Tuition Currency
            \item Endowment Currency
        \end{itemize} 
        \\
        \bottomrule
    \end{tabularx}
\vskip -0.1in
\end{table}

\begin{figure}[t]
    \centering
    \includegraphics[width=1.02\linewidth]{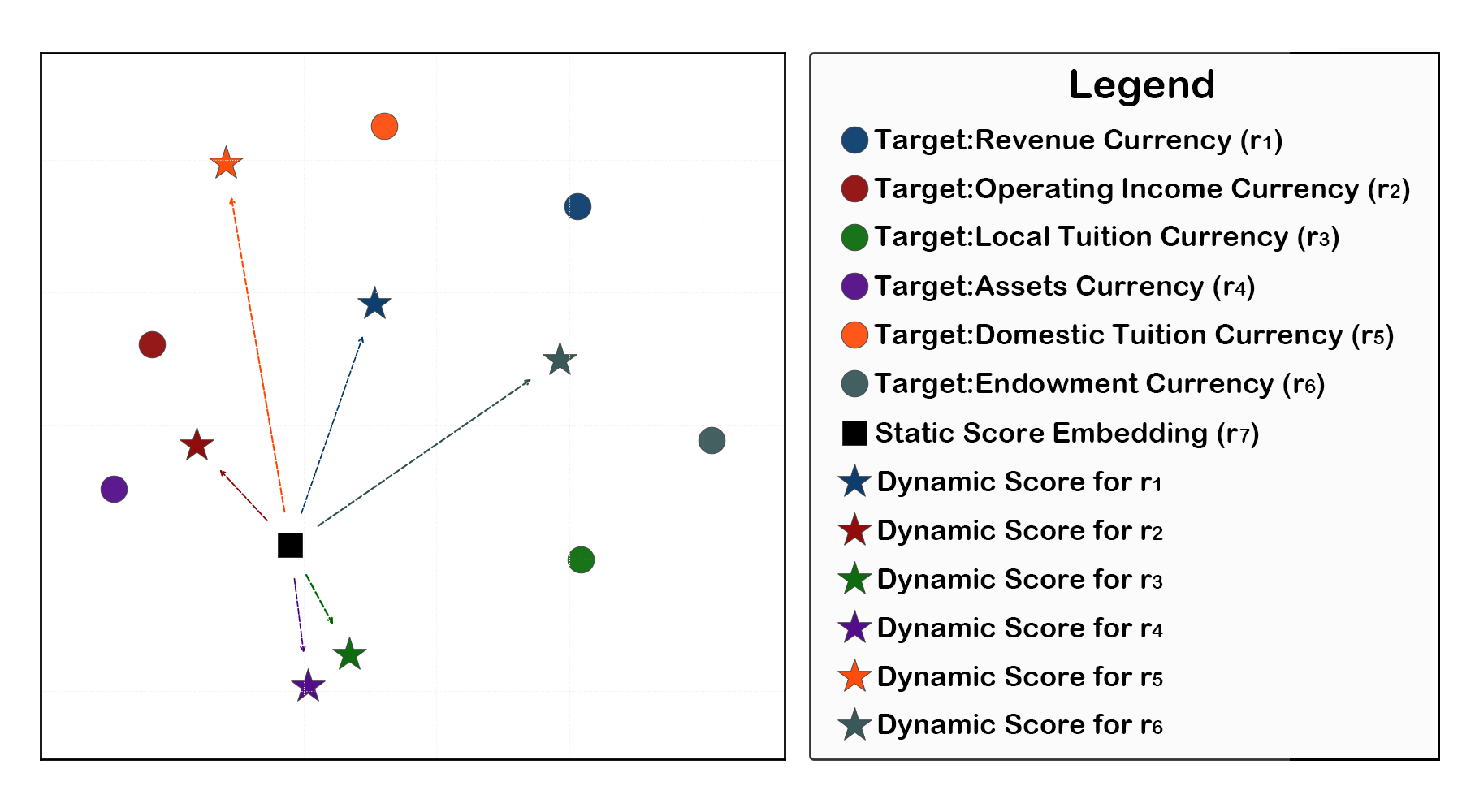}
    \vskip -0.15in
    \caption{T-SNE visualization of Static Score versus Flow-Modulated Score representations on FB15k-237.}
    \label{fig:tsne_fms}
\end{figure}

\section{Entity Prediction Experiment}
\label{sec5}

Building on its demonstrated superiority in relation prediction, we further validate FMS's versatility and power on the fundamental task of entity prediction.
As established previously, these two tasks are inter-reducible. 
To formally adapt our model for this task, we model the probability distribution over all candidate entities. 
Specifically, for a tail prediction query $(h, r, ?)$, the probability of an entity $t$ is given by applying a SoftMax function over the scores derived from our model FMS:
\begin{equation}
    p(t|h,r) = \text{SoftMax} \left( s_{(h, t)}\right).
\end{equation}
Head prediction $(?, r, t)$ is handled symmetrically.

It is crucial to note that this adaptation necessitates a change in the training paradigm. 
Unlike in relation prediction where the target relation $r$ is masked from the input pair $(h, t)$, the model for entity prediction must be trained with the relation $r$ as part of the input condition to properly predict the target entity. 
This rigorous approach allows us to confirm whether the superior performance of FMS extends to this fundamental KG completion task.

\begin{table*}[t]
  \centering
  \small
    \setlength{\tabcolsep}{5.8pt}
  \renewcommand{\arraystretch}{1.1}
  \caption{Transductive results. Performance comparison on FB15k-237, NELL995, Kinship and UMLS datasets. Best results are in \textbf{bold}, second best are \underline{underlined}.}
  \vskip -0.1in
  \label{tab:entity_results}
  \begin{tabular}{c|c|ccc|ccc|ccc|ccc}
    \toprule
    \multirow{2}{*}{\textbf{Type}} & \multirow{2}{*}{\textbf{Model}} & \multicolumn{3}{c|}{\textbf{FB15k-237}} & \multicolumn{3}{c|}{\textbf{NELL995}} & \multicolumn{3}{c|}{\textbf{Kinship}}  & \multicolumn{3}{c}{\textbf{UMLS}}\\
    \cmidrule(lr){3-5} \cmidrule(lr){6-8} \cmidrule(lr){9-11} \cmidrule(lr){12-14}
                                & & \textbf{MRR} & \textbf{H@1} & \textbf{H@3} & \textbf{MRR} & \textbf{H@1} & \textbf{H@3} & \textbf{MRR} & \textbf{H@1} & \textbf{H@3} & \textbf{MRR} & \textbf{H@1} & \textbf{H@3} \\
    \midrule
    \multicolumn{1}{c|}{\multirow{6}{*}{\textbf{Emb-based}}} & TransE                 & 28.9 & 19.8 & 32.4          & 32.6 & 22.5    & 39.9             & 26.3 & 0.5  & 41.8           & 24.7 & 7.3 & 29.7 \\
      & ComplEx                   & 24.7 & 15.8 & 27.5          & 23.7 & 19.7    & 25.2             & 41.8 & 24.2 & 49.9          & 41.1 & 27.3 & 46.8 \\
        & RotaE                   & 33.7 & 24.1 & 37.5          & 35.6 & 29.8    & 39.5             & 65.1 & 50.4 & 75.5          & 74.4 & 63.6 & 82.2 \\
   & ConvE                  & 32.5 & 23.7 & 35.6          & 51.1 & 44.6    & -             & 68.5 & 55.2 & 78.6          & 75.6 & 69.7 & 80.7 \\
  & TDN                  & 35.0 & 26.3 & 39.5          & - & -    & -             & 78.0 & 67.7 & 86.7          & 86.0 & 82.4 & 87.3 \\
    \midrule
    \multicolumn{1}{c|}{\multirow{2}{*}{\textbf{Rule-based}}}
    & RNNLogic                 & 34.4 & 25.2 & 38.0          & 41.6 & 36.3   & -             & 72.2 & 59.8 & 81.4          & 84.2 & 77.2 & 89.1 \\
   & DRUM                   & 23.8 & 17.4 & 26.1          & 53.2 & 46.0    & -             & 33.4 & 18.3 & 37.8          & 54.8 & 35.8 & 69.9 \\
     \midrule
    \multicolumn{1}{c|}{\multirow{4}{*}{\textbf{GNN-based}}}
   & PathCon                  & 50.6 & 39.6 & 56.3          & 49.5 & 44.4   & 51.9             & 78.8     & 65.6    & \underline{91.5}             & 90.1     & 85.7   & 93.2 \\
   & NBFNet                   & 41.5 & 32.1 & 45.4          & 52.5 & 45.1   & -             & 60.6 & 43.5 & 72.5          & 77.8 & 68.8 & 84.0 \\
       & RED-GNN
       & 38.2 & 31.7 & 42.9 & 54.9 & 46.3    & \underline{60.9}             & 78.5    & 69.0    & 85.4             & 96.3     & 94.5    & \underline{97.6} \\
    & AdaProp                  & 42.5 & 31.0 & 47.0 & \underline{55.2} & \underline{49.7}    & 58.6            & 83.6    & 74.9    & 91.1            & \underline{96.4}   & \underline{94.6}    & \textbf{97.9} \\
    \midrule
    \multicolumn{1}{c|}{\multirow{3}{*}{\textbf{Diffusion-based}}}
    & FDM                    & 48.5 & 38.6 & 52.9          & - & -    & -             & \underline{83.7} & \underline{76.1} & 89.7 & 92.2 & 89.3 & 94.4 \\
    & KGDM                   & \underline{52.0} & \underline{42.3} & \underline{56.6} & - & -    & -             & 78.9 & 68.7 & 87.0          & 90.9 & 87.2 & 93.7 \\
    & DiffusionE             & 37.6 & 29.4 & -         & 52.2 & 49.0   & -            & -     & -    & -             & \textbf{97.0} & \textbf{95.7} & - \\
    \midrule
    \multicolumn{1}{c|}{\textbf{Ours}} & \textbf{FMS}  & \textbf{67.2} & \textbf{57.4} & \textbf{72.9}& \textbf{60.6} & \textbf{54.9} & \textbf{65.8} & \textbf{99.3} & \textbf{98.7} & \textbf{99.9} & 92.5 & 89.0 & 94.8 \\
    \bottomrule
  \end{tabular}
\end{table*}

\subsection{Transductive Settings}
\subsubsection{Datasets}
We conduct experiments on four knowledge graph datasets: (i) FB15k-237 (see \textcolor{red}{Section} \ref{sec4.1} for details); (ii) NELL995, specifically an extended version (see \textcolor{red}{Section} \ref{sec4.1} for background); (iii) Kinship, a knowledge graph describing kinship relations, focusing on family and blood relations~\cite{hinton1986kinship}; (iv) UMLS, a biomedical knowledge graph encompassing medical terminology, concepts, and their interrelationships~\cite{bodenreider2004unified}. The statistics above are summarized in \autoref{tab:all_dataset_stats_ent}.

\subsubsection{Baselines}
 We categorize all baseline methods into four groups: \textbf{(i) Embedding-based:} TransE~\cite{bordes2013translating}, ComplEx~\cite{trouillon2016complex}, RotatE~\cite{sun2019rotate}, ConvE~\cite{dettmers2018convolutional}, TDN~\cite{wang2023tdn}, FDM~\cite{long2024fact}; \textbf{(ii) Rule-based:} RNNlogic~\cite{qu2020rnnlogic}, DRUM~\cite{sadeghian2019drum}; \textbf{(iii) GNN-based:}Pathcon~\cite{wang2021relational}, NBFNet~\cite{NBFnet2021}, RED-GNN~\cite{zhang2022knowledge}, AdaProp~\cite{zhang2023adaprop};
 \textbf{(iv) Diffusion-Based:} FDM~\cite{long2024fact}, KGDM~\cite{Long_Zhuang_Li_Wei_Li_Wang_2024}, DiffusionE~\cite{cao2024diffusione}.

\subsubsection{Evaluation Metrics}
To ensure a fair comparison, our training and evaluation
setup follows that of RED-GNN~\cite{zhang2022knowledge}\footnote{ \url{https://github.com/LARS-research/RED-GNN}}. For evaluation, each test triplet $(h, r, t)$ is used to create two entity prediction queries: a tail prediction query $(h, r, ?)$ and a head prediction query $(?, r, t)$. 
The objective is to rank the ground-truth entity ($t$ for the former, $h$ for the latter) against a set of candidate entities. 
These candidates are generated by corrupting the triplet; for instance, in tail prediction, the ground truth $t$ is replaced by every other entity $t'$ in the knowledge graph to form negative samples $(h, r, t')$. 
We report Mean Reciprocal Rank (MRR) and Hits@N under the standard filtered setting~\cite{sun2019rotate}, which ensures that any other valid triplets appearing in the candidate set are removed before ranking.

\subsubsection{Implementation Details}

Our experimental setup is largely consistent with the details provided in \textcolor{red}{Section}~\ref{sec4.1}. 
However, we make specific adjustments to three key hyperparameters: the number of context hops, the number of neighbor samples, and the Top-K value. 
This modification is motivated by the significant disparity between the number of entities and relations; for instance, FB15k-237 contains 14,541 entities but only 237 relations. 
Modeling the probability distribution over a much larger set of entities necessitates richer semantic information. 
Therefore, we uniformly set the number of context hops to 2 and the Top-K value to 10 across all datasets. 
The number of neighbor samples is generally set to 16, with a specific exception for the NELL-995 dataset, for which we use a value of 10.

\subsection{Inductive Settings}

\subsubsection{Datasets}
Building upon this, we proceed with four subsets each from the FB15k237 and NELL-995 datasets, for a detailed examination of the dataset specifics, refer to \textcolor{red}{Section}~\ref{sec4.2}.

\subsubsection{Baselines}
For inductive learning, traditional Embedding-based methods are not suitable for the task due to their inability to generalize to unseen entities, we choose the following two classes of classical and effective models: 
\textbf{(i) Rule-based:} RuleN~\cite{meilicke2018fine}, NeuralLP~\cite{yang2017differentiable}, DRUM~\cite{sadeghian2019drum};
\textbf{(ii) GNN-based:} GraIL~\cite{teru2020inductive}, PathCon~\cite{wang2021relational}.

\subsubsection{Baselines}

Since completion methods that learn entity embeddings during training cannot work in this setting, we select baselines from the following three categories:
\textbf{(i) Rule-based methods:} NeuralLP~\cite{yang2017differentiable}, DRUM~\cite{sadeghian2019drum}, and RuleN~\cite{meilicke2018fine}.
\textbf{(ii) GNN-based methods:} GraIL~\cite{teru2020inductive}, PathCon~\cite{wang2021relational}, NBFNet~\cite{NBFnet2021}, RED-GNN~\cite{zhang2022knowledge}, and AdaProp~\cite{zhang2023adaprop}.
\textbf{(iii) Diffusion-based methods:} DiffusionE~\cite{cao2024diffusione}.
It is worth noting that KGDM and FDM use an embedding-based approach to encode entities and relations; therefore, they are not considered in the inductive learning setting.

\subsubsection{Evaluation Metrics}
The setup and parameter selection are consistent with the transduction entity prediction experiments.

\begin{table}[t] 
    \centering
    \caption{MRR comparison for entity prediction by multiple relation category on FB15k-237.}\label{tab:entity_relation_category}
    \vskip -0.1in
    \setlength{\tabcolsep}{7pt}
    \begin{tabular}{lccccc}
    \toprule
    &  \textbf{TransE} &  \textbf{RotatE} &  \textbf{NBFNet} & \textbf{FDM} & \textbf{FMS}\\
    \midrule
        \textbf{Head Pred} & & & & & \\
        \multicolumn{1}{c}{1-N}  & 7.9 & 8.1 & 16.5 & \underline{20.3} & \textbf{39.8} \\
        \multicolumn{1}{c}{N-1}  & 45.5 & 46.7 & 49.9 & \underline{55.9} & \textbf{69.0} \\
        \multicolumn{1}{c}{N-N}  & 22.4 & 23.4 & 34.8 & \underline{42.3} & \textbf{67.7}\\
    \midrule
    \textbf{Tail Pred} & & & & & \\
        \multicolumn{1}{c}{1-N}  & 74.4 & 74.7 & 79.0 & \underline{82.6} & \textbf{85.1} \\
        \multicolumn{1}{c}{N-1}  & 7.1 & 7.0 & 12.2 & \underline{16.7} & \textbf{33.9} \\
        \multicolumn{1}{c}{N-N}  & 33.0 & 33.8 & 45.6 & \underline{54.3} & \textbf{66.5} \\
    \bottomrule
    \end{tabular}
  \vskip -0.1in
\end{table}

\begin{table}[t]
  \centering
  \caption{Statistics of all datasets. $\mathbb{E}[d]$ and $\mathrm{Var}[d]$ are mean and variance of the node degree distribution, respectively.}
\vskip -0.1in
  \label{tab:all_dataset_stats_ent}
  \renewcommand{\arraystretch}{1.1}
  \setlength{\tabcolsep}{3.5pt}
  \begin{tabular}{c | c c c c c c c}
    \toprule
    \textbf{Dataset} & \textbf{Nodes} & \textbf{Relations} & \textbf{Train} & \textbf{Val} & \textbf{Test} & \textbf{$\mathbb{E}[d]$} & \textbf{$\mathrm{Var}[d]$} \\
    \midrule
    NELL995   & 74.5k & 200 & 37.4k & 543  & 2.8k & 2.6 & 124.7 \\
    Kinship   & 104   & 25  & 3.2k  & 2.1k & 5.3k & 205.5 & 2.8 \\
    UMLS      & 135   & 46  & 1.3k  & 569  & 633  & 37.4 & 1205.3 \\
    \bottomrule
  \end{tabular}
    \vskip -0.1in
\end{table}
% Inductive

\begin{table*}[b]
  \centering
  \footnotesize
  \caption{Inductive results. The best results are shown in bold while the second-best results are shown underlined.}
\vskip -0.1in
  \label{tab:entity_inductive_results_mrr_h10}
  \renewcommand{\arraystretch}{1.3}
  \setlength{\tabcolsep}{2.5pt}
  \begin{tabular}{c|c|cc|cc|cc|cc|cc|cc|cc|cc}
    \toprule
    \multirow{3}{*}{\textbf{Type}} & \multirow{3}{*}{\textbf{Model}} & \multicolumn{8}{c|}{\textbf{FB15k-237}} & \multicolumn{8}{c}{\textbf{NELL-995}} \\
    \cmidrule(lr){3-10} \cmidrule(lr){11-18}
    & & \multicolumn{2}{c|}{\textbf{v1}} & \multicolumn{2}{c|}{\textbf{v2}} & \multicolumn{2}{c|}{\textbf{v3}} & \multicolumn{2}{c|}{\textbf{v4}} & \multicolumn{2}{c|}{\textbf{v1}} & \multicolumn{2}{c|}{\textbf{v2}} & \multicolumn{2}{c|}{\textbf{v3}} & \multicolumn{2}{c}{\textbf{v4}} \\
    \cmidrule(lr){3-4} \cmidrule(lr){5-6} \cmidrule(lr){7-8} \cmidrule(lr){9-10} \cmidrule(lr){11-12} \cmidrule(lr){13-14} \cmidrule(lr){15-16} \cmidrule(lr){17-18}
    & & \textbf{MRR} & \textbf{H@10} & \textbf{MRR} & \textbf{H@10} & \textbf{MRR} & \textbf{H@10} & \textbf{MRR} & \textbf{H@10} & \textbf{MRR} & \textbf{H@10} & \textbf{MRR} & \textbf{H@10} & \textbf{MRR} & \textbf{H@10} & \textbf{MRR} & \textbf{H@10} \\
    \midrule
    \multicolumn{1}{c|}{\multirow{3}{*}{\centering \textbf{Rule-based}}} 
    & Neural-LP &32.5 & 46.8 &38.9 & 58.6 & 40.0 & 57.1 &39.6 & 59.3 &61.0 & 87.1 &36.1 & 56.4 &36.7 & 57.6 &26.1 & 53.9\\
    & DRUM      &33.3 & 47.4 &39.5 & 59.5 &40.2 & 57.1 &41.0 & 59.3 &62.8 & \underline{87.3} &36.5 & 54.0 &37.5 & 57.7 &27.3 & 53.1 \\
    & RuleN     &36.3 & 44.6 &43.3 & 59.9 &43.9 & 60.0 &42.9 & 60.5 &61.5 & 76.0 &38.5 & 51.4 &38.1 & 53.1 &33.3 & 48.4 \\
    \midrule
    \multicolumn{1}{c|}{\multirow{5}{*}{\centering \textbf{GNN-based}}} 
    & GraIL     & 27.9 & 42.9 & 27.6 & 42.4 & 25.1 & 42.4 & 22.7 & 38.9 & 48.1 & 56.5 & 29.7 & 49.6 & 32.2 & 51.8 & 26.2 & 50.6 \\
    & PathCon   &31.7 &53.9 &33.3 &54.7 &36.1 &54.7 &40.1 &61.3 & - & - & 37.3 & 59.4 & 45.6 & 59.0 & 29.2 & 44.8 \\
    & NBFNet    & 30.7 & 51.7 & 36.9 & 63.9 & 33.1 & 58.8 & 30.5 & 55.9 & 58.4 & 79.5 & 41.0 & 63.5 & 42.5 & 60.6 & 28.7 & \underline{59.1} \\
    & RED-GNN   & \underline{36.9} & 48.3 & 46.9 & 62.9 & 44.5 & 60.3 & 44.2 & 62.1 & 63.7 & 86.6 & 41.9 & 60.1 & 43.6 & 59.4 & 36.3 & 55.6 \\
    & AdaProp   & 31.0 & \underline{55.1} & 47.1 & 65.9 & \underline{47.1} & \underline{63.7} & \underline{45.4} & \underline{63.8} & \underline{64.4} & \textbf{88.6} & \underline{45.2} & 65.2 & 43.5 & 61.8 & \underline{36.6} & \textbf{60.7} \\
    \midrule
        \multicolumn{1}{c|}{\multirow{1}{*}{\centering \textbf{Diffusion-based}}} 
    & DiffusionE &31.0 & 52.5 & \underline{47.4} & \underline{66.1} & 45.6 & 62.7 & 44.6 & 62.5 & \textbf{67.3} & \underline{87.3} & 42.4 & \underline{66.0} & \underline{45.8} & \underline{62.0} & 30.8 & 52.3\\
    \midrule
    \multicolumn{1}{c|}{\multirow{1}{*}{\centering \textbf{Ours}}} 
    & FMS       &\textbf{42.3} &\textbf{63.3} &\textbf{52.8} &\textbf{70.6} &\textbf{53.5} &\textbf{73.3} &\textbf{57.6} &\textbf{75.6} &- &- &\textbf{50.4} &\textbf{68.8} &\textbf{54.7} &\textbf{67.6} &\textbf{39.1} &53.9 \\
    \bottomrule
  \end{tabular}
\end{table*}

\subsection{Overall Performance}

To evaluate FMS on entity prediction, we conducted a parallel set of experiments, with results presented for both transductive and inductive settings.

The transductive results are summarized in \autoref{tab:entity_results}. FMS demonstrates superior performance across the majority of datasets and metrics. 
Specifically, FMS achieves a remarkable \textbf{25.2\% relative} increase in MRR over the strongest baseline on FB15k-237, a \textbf{5.4\% relative} increase on NELL995, and an \textbf{15.6\% relative} increase on Kinship. 
Particularly on the Kinship dataset, FMS achieves a near-perfect MRR of 99.3, showcasing its exceptional modeling capacity on dense, complex graph structures.
While DiffusionE shows the best performance on UMLS, FMS still delivers highly competitive results, establishing itself as a new state-of-the-art model for transductive entity prediction on diverse benchmarks.

To delve deeper into its capabilities, we analyze FMS's performance on different relation categories within FB15k-237, as detailed in \autoref{tab:entity_relation_category}. FMS achieves overwhelmingly superior performance across all prediction scenarios (1-N, N-1, N-N for both Head and Tail prediction). 
The improvements are especially pronounced in more challenging categories. For instance, in the N-1 tail prediction task, FMS more than doubles the MRR of the best baseline, increasing it from 16.7 to 33.9. 
This uniform dominance indicates that FMS successfully captures the intricate and varied semantic mappings between entities, demonstrating its exceptional ability to handle complex reasoning scenarios across all relation types.

The inductive entity prediction results are presented in \autoref{tab:entity_inductive_results_mrr_h10}. 
FMS consistently and significantly outperforms all baselines across all versions of the FB15k-237 and NELL-995 datasets, often by a large margin. 
For example, on FB15k-237 v4, FMS achieves an MRR of 57.6, surpassing the next-best model's 45.4. Similarly, on NELL-995 v3, FMS's MRR of 54.7 is substantially higher than the runner-up's 45.8.
Furthermore, the model's strong and stable performance across different subgraph versions (v1 to v4) highlights its excellent robustness and generalization capabilities in unseen entity settings. This consistent and significant outperformance across diverse settings underscores the powerful generalization capabilities of FMS in inductive learning.

\section{Conclusion}

In this work, we introduced Flow-Modulated Scoring (FMS), a novel framework designed to address the limitations of existing knowledge graph completion methods in modeling the semantic contextual dependency and dynamic nature of relations. FMS innovatively combines a semantic context learning module, which generates context-aware entity embeddings, with a condition flow-matching module that learns the evolution dynamics between entities based on this context. By using the predicted vector field to dynamically modulate a base static score, FMS effectively synergizes rich static representations with conditioned dynamic flow information. Our comprehensive experiments on several well-established benchmarks demonstrate that FMS achieves state-of-the-art performance, validating its capability to capture more comprehensive relational semantics and advance the field of knowledge graph completion.

Despite its empirical success, a primary limitation of FMS lies in its significant reliance on the semantic quality of relational features. The model's performance is intrinsically tied to its ability to capture expressive relational patterns. Consequently, on large-scale knowledge graphs characterized by a vast and diverse set of relations—particularly where some relations may be sparse or semantically ambiguous—the efficacy of FMS may be constrained.

To address this limitation, a promising avenue for future research is the incorporation of entity-level features. By jointly modeling both entity and relational semantics, we can potentially achieve a more holistic and robust representation of factual knowledge. However, this extension introduces a non-trivial trade-off: integrating rich entity features would significantly increase the model's parameter count and computational overhead. Therefore, a key challenge will be to strike a balance between enhanced expressiveness and model efficiency. This necessitates the exploration of more scalable encoding and aggregation mechanisms, such as parameter-sharing techniques or lightweight attention modules, to create a more comprehensive yet practical model.

\bibliographystyle{IEEEtran}
\bibliography{sample-base}

\begin{IEEEbiography}[{\includegraphics[width=0.9in,height=1.25in,clip,keepaspectratio]{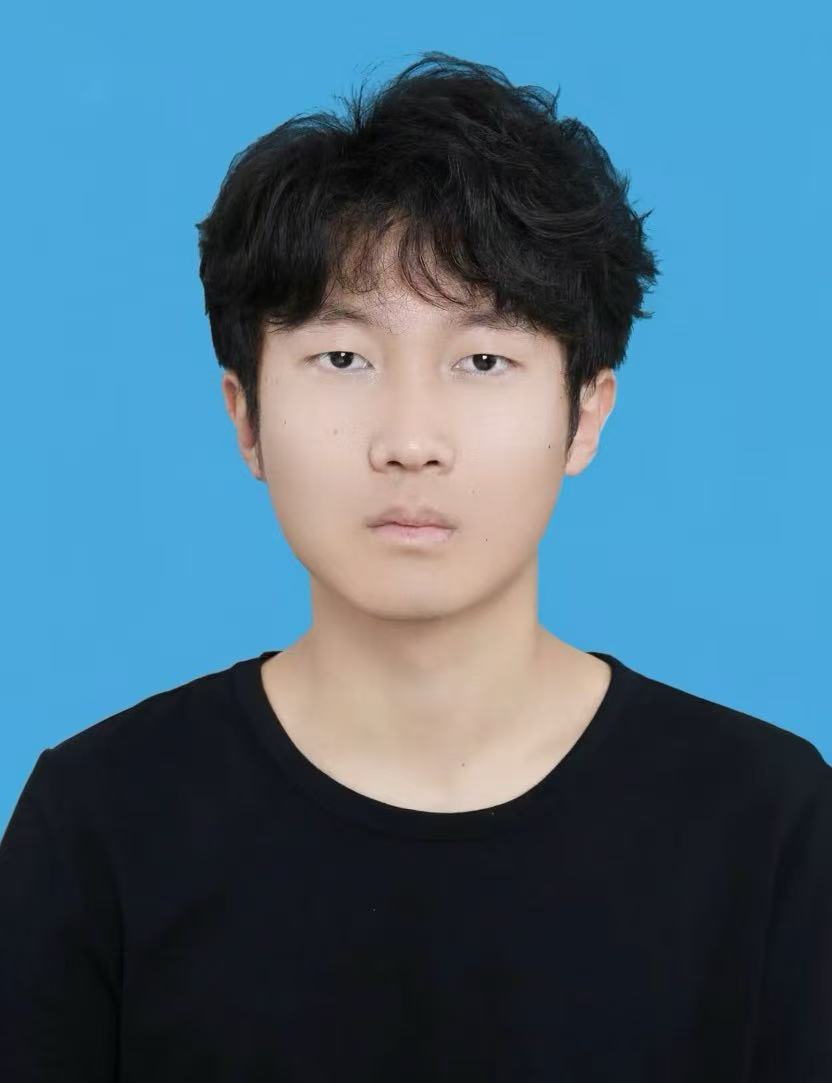}}]{Siyuan Li}
is currently a senior undergraduate student in Computer Science and Technology at Dalian University of Technology. He is also currently an intern at the School of Artificial Intelligence, Nanjing University. His research interests include knowledge representation and reasoning, data mining, and human-computer interaction and visualization.
\end{IEEEbiography}

\begin{IEEEbiography}[{\includegraphics[width=0.9in,height=1.25in,clip,keepaspectratio]{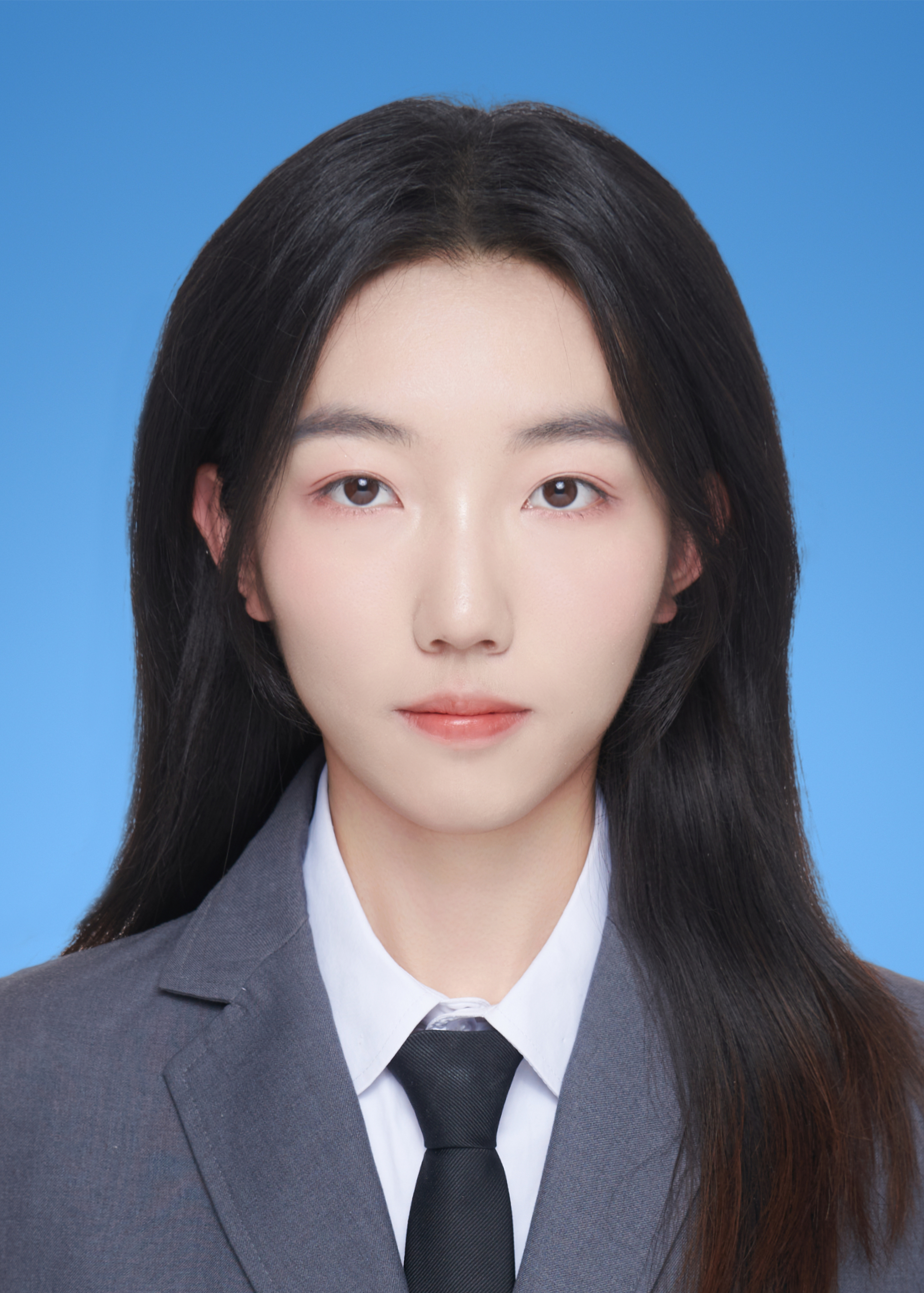}}]{Ruitong Liu}
is currently a senior undergraduate student majoring in Mathematics at Dalian University of Technology. She has been admitted to pursue a postgraduate degree at Peking University. Her research interests include large language models, data-centric AI, and knowledge graphs.
\end{IEEEbiography}

\begin{IEEEbiography}[{\includegraphics[width=0.9in,height=1.25in,clip,keepaspectratio]{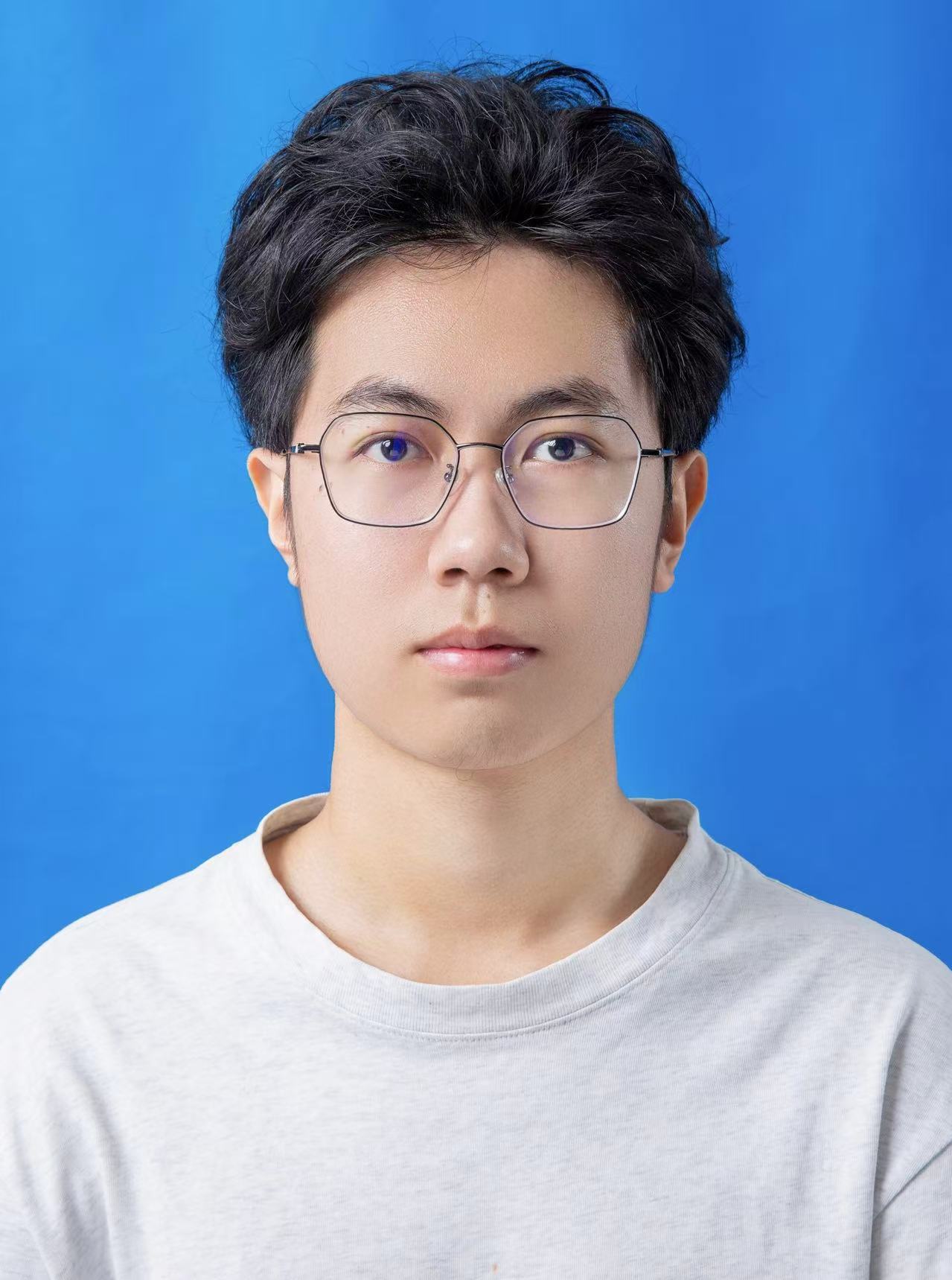}}]{Yan Wen}
is currently a senior undergraduate student in Vehicle Engineering at Beijing Institute of Technology. He is conducting research on autonomous driving, machine learning, and multimodal prediction and planning.\end{IEEEbiography}

\begin{IEEEbiography}[{\includegraphics[width=0.9in,height=1.25in,clip,keepaspectratio]{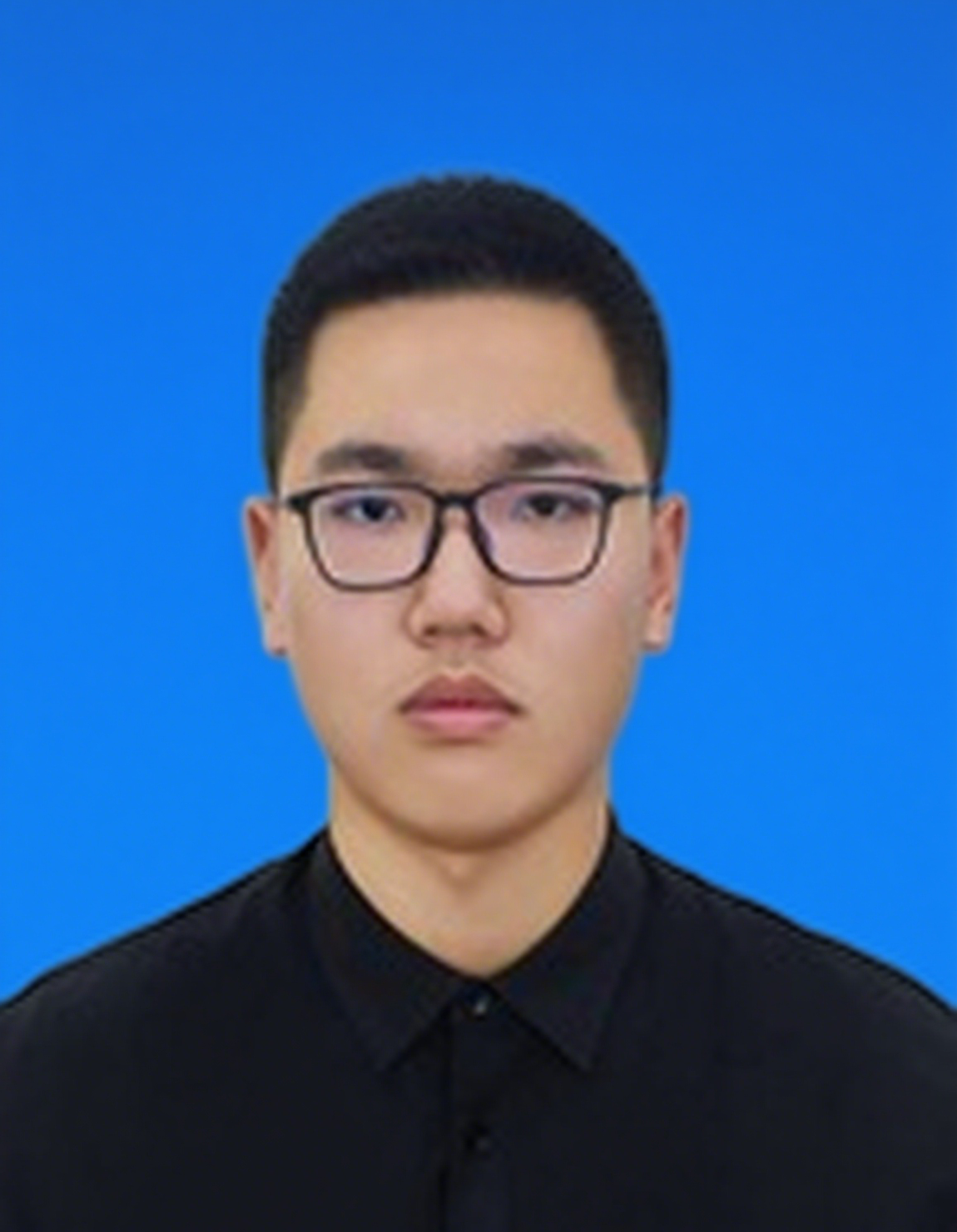}}]{Te Sun}
is a senior undergraduate student of School of Automation Science and Engineering, Dalian University of Technology.He is interested in temporal knowledge graph question answering and robust knowledge graph reasoning. He also has a keen interest in ai agent for science.
\end{IEEEbiography}

\begin{IEEEbiography}[{\includegraphics[width=0.9in,height=1.25in,clip,keepaspectratio]{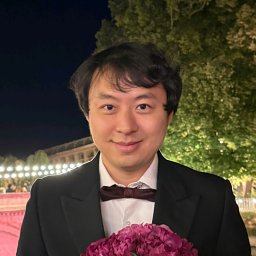}}]
{Andi Zhang}
is a Machine Learning Research Associate at the Centre for AI Fundamentals, University of Manchester, U.K. He received the B.Sc. degree from the University of Manchester, and the M.Phil., M.A.St., and Ph.D. degrees from the University of Cambridge. His/Her research interests include generative AI, machine learning, and AI fundamentals.
\end{IEEEbiography}

%\vspace{-1cm}
\begin{IEEEbiography}[{\includegraphics[width=0.9in,height=1.25in,clip,keepaspectratio]{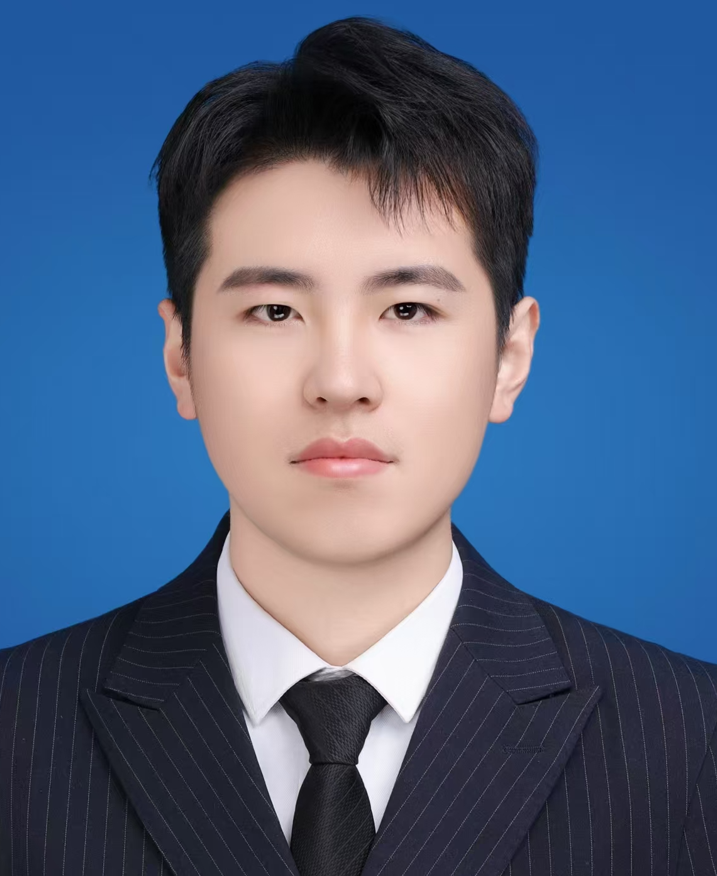}}]{Yanbiao Ma}
received his B.Eng. in Intelligent Science and Technology from Xidian University in 2020 and his Ph.D. in Computer Science and Technology from Xidian University in 2025. He is currently an Assistant Professor at the Gaoling School of Artificial Intelligence, Renmin University of China. His research focuses on multimodal large models, computational neuroscience, long-tailed learning, and federated learning.
\end{IEEEbiography}

\begin{IEEEbiography}[{\includegraphics[width=1in,height=1.25in,clip,keepaspectratio]{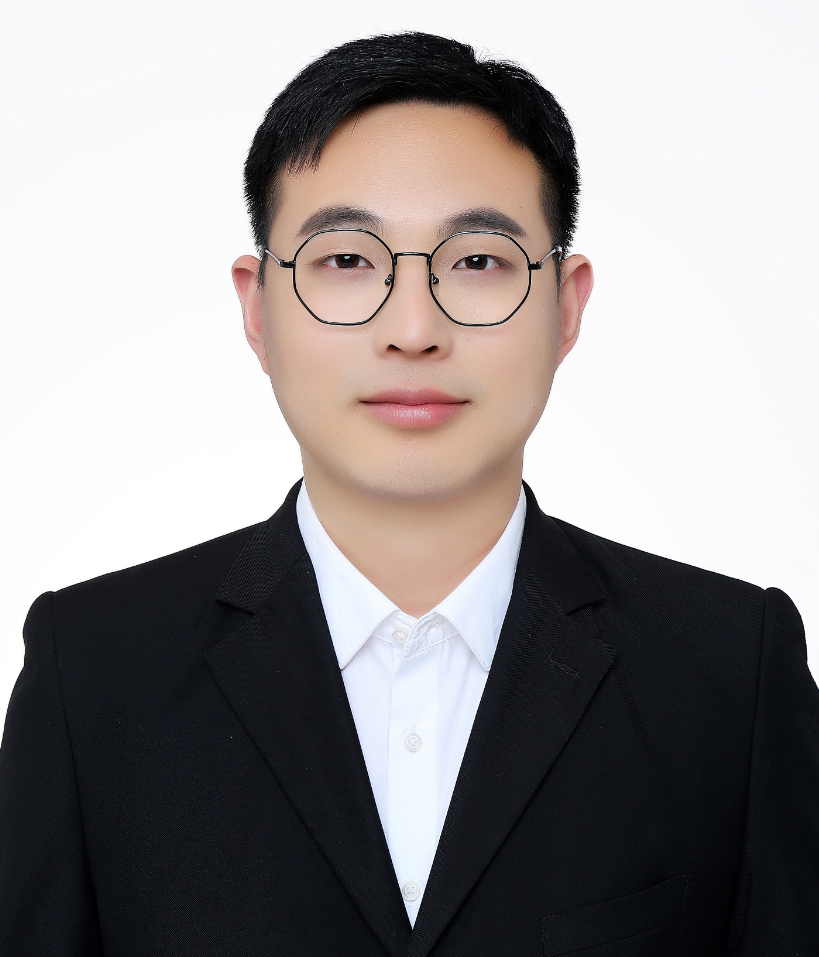}}]{Xiaoshuai Hao}
		received his Ph.D. from the Institute of Information Engineering, Chinese Academy of Sciences, in 2023. He is currently a researcher at the Beijing Academy of Artificial Intelligence, specializing in embodied multimodal large models. His research interests encompass embodied intelligence, multimodal learning, and autonomous driving. Dr. Hao has published over 40 papers in top-tier journals and conferences, including TIP, Information Fusion, NeurIPS, ICLR, ICML, CVPR, ICCV, ECCV, ACL, AAAI, and ICRA. He has achieved outstanding results in international competitions, securing top-three placements at prestigious conferences like CVPR and ICCV. Additionally, he serves on the editorial board of Data Intelligence and is an organizer for the RoDGE Workshop at ICCV 2025 and The RoboSense Challenge at IROS 2025.
\end{IEEEbiography}



%\bibliographystyle{IEEEtran}


%\vfill


\end{document}